\pdfoutput=1
\documentclass[letterpaper]{article}
\usepackage[totalwidth=480pt, totalheight=680pt]{geometry}

\usepackage{authblk}

\usepackage{algorithm}
\usepackage{algorithmicx}
\usepackage{algpseudocode}
\usepackage{subcaption}
\usepackage{enumerate}
\usepackage{bbm}
\usepackage{amsmath}
\usepackage{amssymb}
\usepackage{amsthm}
\usepackage{pifont}
\usepackage{booktabs}
\usepackage{xcolor}
\definecolor{darkgreen}{rgb}{0,0.5,0}
\usepackage{hyperref}
\hypersetup{
    unicode=false,          
    colorlinks=true,        
    linkcolor=red,          
    citecolor=darkgreen,    
    filecolor=magenta,      
    urlcolor=blue           
}
\usepackage[capitalize, nameinlink]{cleveref}

\usepackage{thm-restate}
\theoremstyle{plain}
\newtheorem{theorem}{Theorem}

\newtheorem{example}[theorem]{Example}
\theoremstyle{definition}

\theoremstyle{remark}

\usepackage{tikz}
\usetikzlibrary{calc, graphs, graphs.standard, shapes, arrows, arrows.meta, positioning, decorations.pathreplacing, decorations.markings, decorations.pathmorphing, fit, matrix, patterns, shapes.misc, tikzmark}

\newcommand{\bA}{\mathbf{A}}
\newcommand{\bK}{\mathbf{K}}
\newcommand{\bP}{\mathbf{P}}
\newcommand{\bQ}{\mathbf{Q}}
\newcommand{\bU}{\mathbf{U}}

\newcommand{\bz}{\mathbf{z}}

\newcommand{\Pone}{\textbf{\texttt{1}}}
\newcommand{\Pzero}{\textbf{\texttt{0}}}

\newcommand{\eps}{\varepsilon}
\newcommand{\E}{\mathbb{E}}
\newcommand{\N}{\mathbb{N}}
\newcommand{\R}{\mathbb{R}}

\DeclareMathOperator*{\argmax}{argmax}

\title{On Sequential Fault-Intolerant Process Planning}
\author[1]{Andrzej Kaczmarczyk\thanks{Equal contribution}}
\author[2]{Davin Choo$^*$}
\author[3]{Niclas Boehmer$^*$}
\author[4]{Milind Tambe}
\author[5]{Haifeng Xu}
\affil[1,5]{Department of Computer Science, University of Chicago}
\affil[2,4]{Harvard University}
\affil[3]{Hasso-Plattner-Institute, University of Potsdam}
\date{}

\begin{document}

\maketitle

\begin{abstract}
We propose and study a planning problem we call \emph{Sequential Fault-Intolerant Process Planning} (SFIPP).
SFIPP captures a reward structure common in many sequential multi-stage decision problems where the planning is deemed successful only if \emph{all} stages succeed.
Such reward structures are different from classic additive reward structures and arise in important applications such as drug/material discovery, security, and quality-critical product design.
We design provably tight online algorithms for settings in which we need to pick between different actions with unknown success chances at each stage. 
We do so both for the foundational case in which the behavior of actions is deterministic, and the case of probabilistic action outcomes, where we effectively balance exploration for learning and exploitation for planning through the usage of multi-armed bandit algorithms. 
In our empirical evaluations, we demonstrate that the specialized algorithms we develop, which leverage additional information about the structure of the SFIPP instance, outperform our more general algorithm.
\end{abstract}

\section{Introduction}
\label{sec:introduction}

Multi-stage decision making problems are ubiquitous. 
Models for capturing these problems mostly assume additive utility or rewards across stages.
The validity of these models relies on the premise that a fault at one stage does not affect the overall utility ``too much.''

While this is natural for a wide range of applications, in many multi-stage real-world scenarios, \emph{a fault at any single stage can immediately lead to the failure of the overall process}. 
For instance, during the development of quality-critical products, a bad supply choice for any part of the products will lead to a product failure \cite{jai-ben-des:a:supplier-selection}; during automated drug or material discovery which often has multiple stages of design, a wrong choice during any stage will fail the development \cite{sta-dec-kus-others:a:materials-development}; in security-sensitive applications (e.g.,\ transporting goods across regions monitored by adversaries \cite{boe-han-xu-tam:c:escape-sensing-games}), the detection by the adversary at any step will lead to the failure of the entire mission; finally, in many popular computer games (e.g.,\ Super Mario Bros and Frogger), winning is declared only when the player passed all barriers, otherwise the game restarts from the beginning.
A key characteristic of these problems is the element of \emph{fault intolerance}\,---\,that is, a reward is only given when the planner does not encounter any faults during the entire multi-stage process.

This work introduces a natural and basic planning model for such multi-stage fault-intolerant decision-making processes under uncertainty, coined \emph{Sequential Fault-Intolerant Process Planning} (SFIPP).
Each SFIPP process has $m$ stages, and the planner needs to pick an action $i_s\in [k]$ to take at any stage $s \in [m]$, which has some success probability $p_{s,i_s} \in [0, 1]$.
The planner receives reward $1$ if and only if their chosen action at \emph{every} stage succeeded (with probability $p_{s,i_s}$ each), and receives reward $0$ otherwise.
Therefore, an SFIPP process is fully specified by a probability matrix $\bP \in [0, 1]^{m \times k}$ where $\bP_{s,i} = p_{s,i}$.
We study the situation where $\bP$ is unknown, but the planner plays this process repeatedly for $T$ times and can learn $\bP$ over time.
Adopting the convention of online learning, we measure the planner's performance via \emph{regret} -- the difference between the algorithm's summed reward and the optimal reward given $\bP$.

\subsection{Our Contributions}
\begin{enumerate}
    \item We introduce and formalize SFIPP to model sequential multi-stage process planning with fault-intolerant reward structures, which naturally captures a wide range of important real-world problems including product quality control, drug discovery, security and computer games.
    SFIPP is closely related to several prior models, which we compare and discuss in \cref{sec:preliminaries}.
    \item A foundational special case of SFIPP is when actions have \emph{deterministic} success, i.e., $\bP$ is a binary matrix containing only $\{\Pzero, \Pone\}$ entries. For this setting, we design a randomized algorithm whose performance depends directly on the \emph{unknown} number of $\Pzero$ entries in $\bP$, and attains the worst-case optimal expected regret.
    \item 
    In the general SFIPP process with random stage successes, we design an algorithm with a tight regret bound inspired by the classic multi-armed bandit  (MAB) problem.
    Our algorithm's analysis hinges on a lemma that upper bounds the regret in the fault-intolerant setting by the regret encountered in the additive utility case.
    We then study how the planner could leverage possible additional knowledge about the types of stages to provably improve the expected regret.
    \item Finally, we evaluate our algorithms empirically in various SFIPP settings with different $\bP$.
    We observe that our specialized algorithms that exploit knowledge of the $\bP$ matrix outperform the general-purpose algorithm.
\end{enumerate}

\paragraph{Outline of paper.}
After formalizing SFIPP and comparing it with some prior work in \cref{sec:preliminaries}, we study the SFIPP with deterministic success/failure in \cref{sec:deterministic} and the more general probabilistic setting in \cref{sec:probabilistic}, and then develop algorithms with improved regret under knowledge of stage types in \cref{sec:known-stage-types}.
Our algorithms are evaluated empirically in \cref{sec:experiments} before we conclude in \cref{sec:conclusion} with some future directions.
Some proof details are deferred to the appendix.

\paragraph{Notation.}
For any natural number $n\in \N_{>0}$, we let $[n]:=\{1,\dots, n\}$.
We use $\mathbbm{1}_{\text{predicate}}$ to denote indicator variables that are 1 if the predicate is true and 0 otherwise.
We also employ standard asymptotic notations.

\section{Formalizing the SFIPP Problem}
\label{sec:preliminaries}

A Sequential Fault Intolerant Process Planning (SFIPP) problem has $m$ \emph{stages} and the planner needs to pick an \emph{action} $i \in [k]$ for each stage. 
At each stage $s \in [m]$, action $i \in [k]$ succeeds with some probability $p_{s,i} \in [0,1]$.
Hence, an SFIPP instance is fully characterized by a probability matrix $\bP$ with $\bP_{s,i} = p_{s,i} \in [0,1]$.
Crucially, succeeding at any stage is independent of any successes at earlier stages.
If $\bP$ is known in advance, this is a trivial planning problem as the planner should always pick $\argmax_{i \in [k]} p_{s,i}$ at any stage~$s$. 

In this paper, we study the situation with \emph{unknown} $\bP$. Hence, the planner needs to learn to plan on the fly, and carefully balance exploration and exploitation.
The player plays the SFIPP process over a horizon of $T \in \N_{> 0}$ rounds. We denote by $i_{t,s} \in [k]$ the action picked by the planner in stage $s \in [m]$ of round $t \in [T]$.
This results in a sequence of $m$ actions $(i_{t,1}, \ldots, i_{t,m})$ across all $m$ stages from which we receive a reward of $1$ for round $t \in [T]$ with probability $\prod_{s=1}^m p_{s,i_{t,s}}$, and $0$ otherwise.
In more detail, we generate an independent binary outcome at each stage and then return the product of these outcomes as the reward, i.e., the outcome at stage $s$ in step $t$ is $1$ (``success''/``pass'') with probability $p_{s,i_{t,s}}$ and $0$ (``failure'') otherwise.
If the process fails and returns a reward of $0$, we observe the \emph{first} stage $s \in [m]$ that the process failed at, i.e., the minimal $s$ for which the generated outcome was $0$. Our rationale for this is that in real-world applications if one encounters failure at one stage, one does not move on to (and observes the outcome of) the next stage, as the process has already failed and needs to be restarted.
In sum, at each step $t\in T$ the planner either knows that all $m$ stages succeeded, or is told the first stage at which their selected action failed.
The planner's goal is to design an algorithm that maximizes the total reward accumulated across all $T$ rounds.

We assume that $T \geq k$, as otherwise the problem is uninteresting since one is not guaranteed to even play the optimal action in the worst case.
Furthermore, as common in the literature, we assume that $k$ and $m$ are constants while $T$ is treated as a variable, and we analyze how our algorithms perform over time as a function of $T$.

Formally, we measure the performance of any algorithm solving SFIPP by its \emph{regret}.
Suppose an algorithm $\textsc{ALG}$ produces a sequence of actions $i_{t,1}, \ldots, i_{t,m}$ for round $t$.
The expected regret for the algorithm $\textsc{ALG}$ is defined as
\begin{equation}
\label{eq:SFIPP-benchmark-objective}
\E[R(\textsc{ALG})]
= \sum_{t=1}^T \left[ \left( \prod_{s=1}^m p_{s, i^*_s} \right) - \left( \prod_{s=1}^m p_{s, i_{t,s}} \right) \right]
\end{equation}
where $i^*_s = \argmax_{i \in [k]} p_{s,i}$ is the best stage $s$ action.

\begin{example}
Consider a single round ($T = 1$) instance with $m = 3$ stages, $k = 4$ actions, and probability matrix
\[
\bP =
\begin{bmatrix}
{\color{red}0.4} & {\color{blue}0.6} & 0.5 & 0.3\\
0.3 & 0.2 & {\color{red}0.1} & {\color{blue}1}\\
0.5 & {\color{blue}0.7} & {\color{red}0.4} & 0.1
\end{bmatrix}
\in \R^{3 \times 4}
\]
In this SFIPP instance, we see that the optimum sequence is $(2,4,2)$ with success probability $p^*_{1,2} \cdot p^*_{2,4} \cdot p^*_{3,2} = {\color{blue}0.6} \cdot {\color{blue}1} \cdot {\color{blue}0.7} = 0.42$.
Meanwhile, the sequence $(1,3,3)$ has success probability $p_{1,1} \cdot p_{2,3} \cdot p_{3,3} = {\color{red}0.4} \cdot {\color{red}0.1} \cdot {\color{red}0.4} = 0.016$ and incurs an expected regret of $0.42 - 0.016 = 0.404$.
\end{example}

\subsection{Connections to Prior Work}
\label{sec:PW}

\textbf{Sequential planning.}
Sequential planning problems are a foundational topic in artificial intelligence (AI) where one is often concerned with finding optimal sequences of actions to transition a system from an initial state to a desired goal state.
Starting as early as $A^*$ search \cite{hart1968formal}, the study of sequential planning evolved to incorporate formal models such as STRIPS \cite{fikes1971strips}, probabilistic reasoning \cite{kushmerick1995algorithm}, and domain-independent heuristics \cite{hoffmann2001ff}.
Modern methods of sequential planning also seek to address computational challenges posed by large state space \cite{bellman1966dynamic}; see also \cite{ghallab2004automated}.
Unfortunately, most sequential planning methods  such as Probabilistic GraphPlan \cite{blum1999probabilistic} assume \emph{known} transition probabilities while the key difficulty in SFIPP lies in the fact that the action probability matrix $\bP$ is \emph{unknown} and has to be learned.

\smallskip
\noindent
\textbf{Reinforcement Learning (RL).}
RL is a sequential decision-making paradigm where agents learn an optimal policy for interacting with an environment. The latter is often modeled as a Markov Decision Process (MDP).
While the SFIPP problem bears similarity to the more general RL framing, RL research often emphasizes convergence analysis. Contrary to that, the objective in SFIPP is to achieve low regret.
While our $m$-stage SFIPP can be framed as an MDP with $\approx m$ states (see, e.g.,\ a survey by \cite{col-kar-sig-oud:a:gcrl-survey}), the customized methods we develop aiming for achieving a provably low regret do not apply to standard MDP formulations.
For instance, our collapsing bandit algorithm in \cref{sec:known-stage-types} enforces that the same action must be chosen for identical stage types, effectively collapsing states and violating the Markov property of MDPs.
However, it is this violation, that lets the algorithm offer both theoretical and empirical improvements in regret bounds over standard approaches.
Furthermore, our work explores SFIPP as a simpler yet powerful specialization of broader models like goal-oriented RL, allowing us to leverage problem-specific insights to design simpler algorithms with strong guarantees, outperforming generic approaches.

\smallskip
\noindent
\textbf{Stochastic multi-arm bandits (MAB).}
Stochastic MABs are closely related to the SFIPP problem.
Concretely, in the special case of only one stage (i.e.\ $m=1$), SFIPP degenerates precisely to the \emph{Bernoulli bandit} problem where the reward of each arm (i.e.\ action in SFIPP) is a Bernoulli distribution.
Therefore, SFIPP is a strict generalization of Bernoulli bandits to many stages.
Note that the binary reward assumption in Bernoulli bandits does not intrinsically simplify the MAB problem as regret lower bounds for MABs all hold for (and in fact, are mostly proven under) Bernoulli bandits, showing that they already contain the most difficult MAB instances \cite{bubeck2012regret,slivkins2019introduction,lattimore2020bandit}.
Therefore, there is strong theoretical evidence to believe that SFIPP is strictly more challenging than standard MAB in the multi-stage setting.

\smallskip
\noindent
\textbf{Combinatorial stochastic bandits.}
The planner picks a sequence of $m$ actions in the SFIPP problem. That is similar to combinatorial bandits, where the learner can pick a subset of actions in one step \cite{cesa2012combinatorial,chen2013combinatorial,combes2015combinatorial,chen2018contextual}.
However, there are two key differences between SFIPP and combinatorial bandits.
First, actions are simultaneously chosen in combinatorial bandits, whereas in SFIPP a failure at one stage effectively stops the planner from playing and learning from any actions for future stages.
The second key difference lies in the reward structure\,---\,in combinatorial bandits, rewards are assumed to be additive (or sometimes submodular) across actions \cite{hazan2012online,qin2014contextual,chen2018contextual}, whereas our reward is a product of outcomes.

\smallskip
\noindent
\textbf{Escape Sensing.}
\cite{boe-han-xu-tam:c:escape-sensing-games} recently introduced and studied 
\emph{escape sensing games} that feature the same reward structure as our SFIPP. 
Motivated by transporting peacekeeping resources by a convoy of ships, they study a model where a planner transports resources through a channel monitored by sensors of an adversary.
A resource is successfully transported if and only if it escapes the sensing of all sensors, just like how the success of an SFIPP round hinges on success at every stage.
In contrast to our work, \cite{boe-han-xu-tam:c:escape-sensing-games} focus entirely on a game-theoretic setting where a strategic adversary chooses the matrix $\bP$ from a certain feasible set. They study the problem of computing equilibria in the game.
Meanwhile, there is no adversary in our setting and our challenge is to learn the \emph{unknown} underlying matrix $\bP$.

\section{SFIPP with  Deterministic Success/Failure}
\label{sec:deterministic}

In this section, we study the foundational special case of SFIPP in which actions are deterministic, i.e.\ the matrix $\bP$ is binary containing only values $\Pzero$ and $\Pone$ and at a stage, an action will either always or never succeed. We say an action $i\in [k]$ is \emph{successful} at a stage $s\in [m]$ if $p_{i,s}=1$ and \emph{failing} otherwise (i.e., $p_{i,s}=0$).
In this setting, the expected regret term in \cref{eq:SFIPP-benchmark-objective} simplifies to
\[
\E[R(\textsc{ALG})]
= \sum_{t=1}^T \left[ \left( \prod_{s=1}^m \mathbbm{1}_{p_{s, i^*_s = 1}} \right) - \left( \prod_{s=1}^m \mathbbm{1}_{p_{s, i_{t,s}} = 1} \right) \right]
\]
Without loss of generality, we may assume that for every stage there is at least one successful action, i.e.\ $\argmax_{i \in [k]} p_{s,i} = 1$ for all $s \in [m]$ (otherwise the regret  becomes $0$).
The challenge now is to spend as few queries as possible to identify a successful action for each stage.
In the rest of this section, we work towards constructing an algorithm that attains the optimal expected regret for this problem.

\begin{theorem}
\label{thm:deterministic-rand-alg}
Consider the SFIPP problem with deterministic success/failure where $z = |\{(s,i) \in [m] \times [k]: p_{s,i} = 0\}|$ denotes the number of zero entries in $\bP$.
Then, there is a randomized algorithm achieving expected regret (\cref{eq:SFIPP-benchmark-objective}) of at most $\frac{z}{2}$.
Furthermore, \emph{any} algorithm incurs an expected regret of at least $\frac{z}{2}$ in the worst case.
\end{theorem}

To show \cref{thm:deterministic-rand-alg}, we first consider the \emph{single-stage} $m = 1$ setting and aim to understand the lower and upper bound on the number of actions that need to be tried to find the first successful action when there are $z$ failing actions for any $0 \leq z \leq k-1$; these are achieved in \cref{lem:one-stage-lower-bound} and \cref{lem:one-stage-upper-bound} respectively.
These bounds will be helpful for the general problem with $m>1$ because the action probabilities across stages are uncorrelated, which means that stages need to be learned separately.
Accordingly, we will extend our one-stage results to the more general $m \geq 1$ setting by first understanding the hardest \emph{multi-stage} distribution of zeroes in \cref{lem:hardest-multi-stage-setup} and then using it to prove that our proposed algorithm \textsc{UniformThenFixed} is optimal in expectation.

Using Yao's lemma \cite{yao1977probabilistic}, we first show a lower bound on the number of queries a randomized algorithm needs to find a successful action in the single-stage case: 
\begin{restatable}{lemma}{onestagelowerbound}
\label{lem:one-stage-lower-bound}
Fix some $z\in \{0,\dots ,k-1\}$ and let $\bA_z$ be the set of all binary arrays of length $k$ with exactly $z$ $\Pzero$ entries. 
In expectation, any randomized algorithm needs to query at least $\frac{z}{k+1-z}$ indices
before it locates a $\Pone$ within an array sampled uniformly at random from $\bA_z$.\footnote{It is natural to wonder whether the lower bound is lower for a deterministic algorithm.
In fact, in the worst case, any deterministic algorithm needs $z+1 \geq \frac{z}{k+1-z} + 1$ queries before identifying the first successful action.
In fact, at the extreme of $z = k-1$, any deterministic algorithm would require $k = z+1$ queries whilst the lower bound in \cref{lem:one-stage-lower-bound} only requires $\frac{k-1}{k+1-(k-1)} + 1 = \frac{z}{2} + 1$ queries, i.e.\ incurring an approximation ratio of 2 since the additive $1$ is always necessary.}
\end{restatable}

As subsequently demonstrated, a simple randomized algorithm that uniformly selects from unknown entries has an optimal query complexity matching the bound from \cref{lem:one-stage-lower-bound}.

\begin{restatable}{lemma}{onestageupperbound}
\label{lem:one-stage-upper-bound}
Fix some $z\in \{0,\dots, k-1\}$ and let $A$ be some binary array of length $k$ with exactly $z$ $\Pzero$s and $k-z$ $\Pone$s.
If we uniformly select an unexplored entry in $A$ until we discover the first $\Pone$ in the array, then in expectation we select $\frac{z}{k+1-z}$ zeros before we encounter a $\Pone$ entry.
\end{restatable}
\begin{proof}[Proof Sketch]
Order the $z$ $\Pzero$s arbitrarily and define $Z_i$ as the indicator whether the $i^{th}$ $\Pzero$ was chosen before \emph{any} $\Pone$.
Then, one can show that $\E \left(\sum_{i=1}^{z} Z_i \right) = \frac{z}{k+1-z}$.
\end{proof}

Using \cref{lem:one-stage-lower-bound,lem:one-stage-upper-bound}, we give the \textsc{UniformThenFixed} algorithm (\cref{alg:uniformthenfixed}) which at each stage uniformly selects an unselected action until it found a successful one which is later always selected.
This algorithm achieves optimal expected regret at each stage and thereby also overall.

\begin{algorithm}[htb]
\begin{algorithmic}[1]
\caption{The \textsc{UniformThenFixed} algorithm.}
\label{alg:uniformthenfixed}
    \Statex \textbf{Input}: Number of rounds $T$, number of stages $m$, number of actions $k$
    \State Initialize an $m \times k$ matrix $\bP$ with all $-1$ entries 
    \Comment{$-1$ denotes ``unexplored''}
    \For{round $t = 1, \ldots, T$}
        \For{stage $s = 1, \ldots, m$}
            \State Let $\bK = \{i \in [k]: \bP_{s,i} = 1\}$
            \If{$\bK \neq \emptyset$}
                \State Pick any action from $\bK$ for stage $s$ in round $t$
                \Comment{Guaranteed success}
            \Else
                \State Pick an action uniformly at random from $\bU = \{i \in [k]: \bP_{s,i} = -1 \}$ and update entry
                $\bP_{s,i}$
                \Statex\hspace{\algorithmicindent}\hspace{\algorithmicindent}\hspace{\algorithmicindent}to $\mathbbm{1}_{\text{$i$ succeeded in stage $s$}}$
            \EndIf
        \EndFor
    \EndFor
\end{algorithmic}
\end{algorithm}

The formal analysis of \textsc{UniformThenFixed} relies on the following lemma which argues that, in terms of expected regret, the hardest distribution of zeroes across multiple stages is one where the zeros are concentrated as much as possible, i.e., having many zeros at a few stages leads to more queries than having a few zeros at every stage.
This is intuitive as $\frac{z}{k+1-z}$ increases as $z$ increases.

\begin{restatable}{lemma}{hardestmultistagesetup}
\label{lem:hardest-multi-stage-setup}
Suppose there are $m \geq 1$ stages and $k$ actions per stage.
Let $z_1, ..., z_m\in \{0,\dots, k-1\}$ such that $z_1 + ... + z_m = z = a(k-1) + b$, for some $a \leq m$ and $0 \leq b < k-1$.
The maximum of $\sum_{s=1}^m \frac{z_s}{k+1 - z_s}$  is $\frac{a}{2} + \frac{b}{k+1-b}$.
\end{restatable}
\begin{proof}[Proof sketch]
Consider the following distribution of $z_1, \ldots, z_m$ into $m$ parts such that $z_1, \ldots, z_m = a(k+1) + b$: $\bz^* = (z^*_1, \ldots, z^*_m) = (\underbrace{k-1, \ldots, k-1}_{a}, b, 0, \ldots,0)$, which concentrates as many counts in the coordinates as possible, up to the constraint of $z_s \leq k+1$.
Such a distribution would correspond to the sum
\[
\sum_{s=1}^m \frac{z^*_s}{k+1-z^*_s}
= \sum_{s=1}^m \frac{z^*_s}{k+1-z^*_s}
= \frac{a}{2} + \frac{b}{k+1-b}
\]
Notice that this expression holds regardless of which index positions within $[m]$ are $0$ as long as the \emph{multiset} of counts across all $m$ indices contain $a$ counts of $(k-1)$, $1$ count of $r$, and the others being $0$.
Using an exchange argument, one can argue that the distribution $\bz^*$ will maximize $\sum_{s=1}^m \frac{z_s}{k+1-z_s}$ under the stated constraints.
\end{proof}

Using \cref{lem:hardest-multi-stage-setup}, we are now ready to give the formal guarantees of our \textsc{UniformThenFixed} algorithm.

\begin{proof}[Proof of \cref{thm:deterministic-rand-alg}]
Consider \textsc{UniformThenFixed} (\cref{alg:uniformthenfixed}) which, at each stage $s \in [m]$, uniformly selects an unselected action until an $\Pone$ is found.
Let $z_1, ..., z_m\in [0,k-1]$, and let $a \in [0,m]$ and $b\in [0,k-1]$ so that $z_1 + ... + z_m = z=a(k-1) + b$. 
Assume that there are $z_s$ failing actions for stage $s$.
We know from \cref{lem:one-stage-upper-bound} that our algorithm will select $\frac{z_s}{k+1-z_s}$ failing actions at stage $s$ before querying a successful action.
It follows that \textsc{UniformThenFixed} incurs an expected regret of $\sum_{s=1}^m \frac{z_s}{k+1 - z_s}$.
By \cref{lem:hardest-multi-stage-setup}, the sum can be upper bounded by $\frac{a}{2} + \frac{b}{k+1-b}$.
Therefore,
\[
\E[R(\textsc{UniformThenFixed})] \leq \frac{a}{2} + \frac{b}{k+1-b} \leq \frac{a}{2} + \frac{b}{2} = \frac{z}{2}
\]

For the lower bound, consider an SFIPP instance when $z_1 = \ldots = z_m = k-1$, i.e.\ $z = m(k-1)$.
By \cref{lem:one-stage-lower-bound}, \emph{any} algorithm \textsc{ALG} solving this instance would select at least $\sum_{s=1}^m \left( \frac{k+1}{k+1-z_s} - 1 \right)$ failing actions.
Therefore,
\[
\E[R(\textsc{ALG})]
\geq \sum_{s=1}^m \frac{z_s}{k+1-z_s}
= \frac{m(k-1)}{2}
= \frac{z}{2}
\]
\end{proof}

\section{SFIPP with Probabilistic Successes}
\label{sec:probabilistic}

We now consider the SFIPP problem in full generality, where actions are probabilistic and $\bP$ is an arbitrary (unknown) matrix in $[0,1]^{m \times k}$.
We develop an algorithm working for arbitrary $\bP$ and in particular do not make any assumptions on the connection between different columns of $\bP$. As discussed in \cref{sec:preliminaries}, this setting constitutes a generalization of multi-armed bandits and we will present a general bandit-based algorithm.
Note that it is not fruitful to directly model SFIPP into a bandit problem since that would result in $k^m$ arms as each action sequence results in a different joint success probability.
At the foundation of our algorithm lies the fact that when success probabilities of actions are uncorrelated across stages, the best we can do is to learn the success probabilities for each stage separately. 
To do so, at each stage of the SFIPP we utilize an arbitrary classic multi-arm bandit algorithm as a black-box. 
Following this reasoning, this section shows:

\begin{theorem}
\label{thm:staged-bandit-guarantees}
Suppose there is a MAB algorithm that achieves a regret bound of $O(R_{T, p_1, \ldots, p_k})$ on the set of Bernoulli arms with success probabilities $p_1, \ldots, p_k$ over a horizon of $T$ rounds.
Given a SFIPP instance with probability matrix $\bP \in \R^{m \times k}$  with rows $\bP_1, \ldots, \bP_m$, 
\Cref{alg:stagedbandit} achieves a regret bound of $O(\sum_{s=1}^m R_{T,\bP_{s}})$.
\end{theorem}

The crux of our analysis relies on the next lemma which will allow us to bound the regret generated in one round of a SFIPP instance (\cref{eq:SFIPP-benchmark-objective}) by the sum of the regret encountered at each stage when interpreting each stage as an individual Bernoulli bandit.

\begin{restatable}{lemma}{prodtosum}
\label{lem:prod-to-sum}
Let $a_1, \ldots, a_m, b_1, \ldots, b_m \in [0,1]$ such that $0 \leq b_i \leq a_i \leq 1$ for all $i \in [m]$.
Then,
\[
0
\leq \left( \prod_{i=1}^m a_i \right) - \left( \prod_{i=1}^m b_i \right)
\leq \sum_{i=1}^m (a_i - b_i)
\]
\end{restatable}
\begin{proof}[Proof sketch]
Perform induction on $m$.
\end{proof}

By applying \cref{lem:prod-to-sum}, any existing classic bandit algorithm can be repeatedly employed at each stage independently in a black-box manner, resulting in a linear multiplicative increase in the regret bound.
We refer to the resulting algorithm as \textsc{StagedBandit} (\Cref{alg:stagedbandit}).

\begin{algorithm}[t]
\begin{algorithmic}[1]
\caption{The \textsc{StagedBandit} algorithm.}
\label{alg:stagedbandit}
    \Statex \textbf{Input}: Number of rounds $T$, number of stages $m$, number of actions $k$, a classic MAB bandit algorithm \textsc{Bandit}
    \State Initialize $m$ \textsc{Bandit} instances $\textsc{Bandit}_1$, $\ldots$, $\textsc{Bandit}_m$, one for each stage
    \For{round $t = 1, \ldots, T$}
        \For{stage $s = 1, \ldots, m$}
            \State Query $\textsc{Bandit}_s$ for an action $i \in [k]$
            \State Use action $i \in [k]$ for stage $s$ in round $t$
            \If{stage $s$ succeeded}
                \State Provide $\textsc{Bandit}_s$ with a reward of 1
            \Else
                \State Provide $\textsc{Bandit}_s$ with a reward of 0
                \State \textbf{break}
                \Comment{Do not perform stages $s+1, \ldots, m$}
            \EndIf
        \EndFor
    \EndFor
\end{algorithmic}
\end{algorithm}

\begin{proof}[Proof of \cref{thm:staged-bandit-guarantees}]
Consider the \textsc{StagedBandit} algorithm (\cref{alg:stagedbandit}) which runs $m$ instances of a MAB algorithm \textsc{Bandit}, one $\textsc{Bandit}_s$ for each stage $s\in [m]$. $\textsc{Bandit}_s$  achieves an instance-dependent regret of $O(R_{T,\bP_s})$ over $T$ rounds on the success probabilities $\bP_s = (p_{s,1}, \ldots, p_{s,k})$ at stage $s$. We can bound the regret of \textsc{StagedBandit} as:
\begin{align*}
&\; \mathbb{E} \left[ R(\textsc{StagedBandit}) \right]\\
= &\; \mathbb{E} \left[ \sum_{t=1}^T \left[ \left( \prod_{s=1}^m p_{s,i^*_s} \right) - \left( \prod_{s=1}^m p_{s,i_{t,s}} \right) \right] \right] \tag{By definition of \cref{eq:SFIPP-benchmark-objective}}\\
\leq &\; \mathbb{E} \left[ \sum_{t=1}^T \sum_{s=1}^m \left( p_{s,i^*_s} - p_{s,i_{t,s}} \right) \right] \tag{By \cref{lem:prod-to-sum} since always $p_{s,i^*_s} \geq p_{s,i_{t,s}}$}\\
= &\; \sum_{s=1}^m \mathbb{E} \left[ \sum_{t=1}^T \left( p_{s,i^*_s} - p_{s,i_{t,s}} \right) \right] \tag{Linearity of expectation}\\
\in &\; O(\sum_{s=1}^m R_{T,\bP_s}) \tag{Since each $\textsc{Bandit}_s$ incurs a regret of $O(R_{T,\bP_s})$}
\end{align*}
Note that for some round $t\in [T]$, not all values of $(i_{t,s})_{s\in [m]}$ may actually be instantiated (since we stop picking actions if we failed at some stage). 
However, one can imagine running $\textsc{Bandit}_s$ until we have actually performed $T$ actions at stage $m$.
Since we are only proving an upper bound, these imaginary regret counts from the later stages do not hurt us in the analysis.
\end{proof}

Note that the inequality of \cref{lem:prod-to-sum} is asymptotically tight when $a_1, \ldots, a_m = 1$ and $b_1, \ldots, b_m = 1 - \eps$ for some $\eps \in o(\frac{1}{m})$.
To see this, consider the Taylor expansion of $(1 - \eps)^m$: $(1 - \eps)^m = 1 - m \eps + \frac{m(m-1)}{2} \eps^2 - O((m \eps)^3) \leq 1 - \frac{m \eps}{2}$ for $\eps \in o(\frac{1}{m})$.
By rearranging, we see that $\left( \prod_{i=1}^m a_i \right) - \left( \prod_{i=1}^m b_i \right) = 1 - (1 - \eps)^m \geq \frac{m \eps}{2} = \frac{1}{2} \sum_{i=1}^m \eps = \frac{1}{2} \sum_{i=1}^m (a_i - b_i)$. Thus our bound from \Cref{thm:staged-bandit-guarantees} is asymptotically tight as well. 

\section{Improved Regret with Known Stage Types}
\label{sec:known-stage-types}

Given additional knowledge about which stages are known to be of the same type, we can modify the algorithm in \cref{sec:probabilistic} to obtain stronger results.
Two stages are said to be of the same type if they share the same optimal action; a special case of this is when there are two identical stages resulting in two identical rows in $\bP$.

We will sketch two ways in which \textsc{StagedBandit} can be modified to ``merge'' stages of the same type to provably attain lower regret and defer pseudocode and theoretical guarantees to \cref{sec:appendix-known-stage-types}.
Our first ``collapsing'' idea is to ``collapse'' all stages of the same type into one single ``meta'' stage, which means that we will always perform the same action throughout all stages of the same type and will only track whether we made it through \emph{all} of the stages of the type.
By doing so, instead of needing to learn the success probability of an action at each stage, we learn 
the success probability of a collapsed action as the product of the success probability across the collapsed stages.
As an example, consider the following SFIPP matrix
\[
\bP =
\begin{bmatrix}
0.9 & 0.8\\
0.6 & 0.5\\
0.5 & 0.7
\end{bmatrix}
\]
We see that the optimal sequence is $(1,1,2)$.
Suppose we now collapse the first two rows of $\bP$ into $\bP'$:
\[
\bP'
=
\begin{bmatrix}
0.9 \cdot 0.6 & 0.8 \cdot 0.5\\
0.5 & 0.7
\end{bmatrix}
=
\begin{bmatrix}
0.54 & 0.4\\
0.5 & 0.7
\end{bmatrix}
\]
In this collapsed matrix $\bP'$, the optimal action sequence is $(1,2)$ which corresponds  to performing $(1,1,2)$ in $\bP$ (as we  perform the same action for all stages of the same type).
As we only collapse rows that have the same optimal action index, optimally solving the collapsed SFIPP problem will recover the optimal action sequence of the original SFIPP.
Furthermore, notice that the probability gap between the best action for the first two stages has widened from $0.9 - 0.8 = 0.1$ and $0.6 - 0.5 = 0.1$ to $0.54 - 0.4 = 0.14$.
The increase in probability gap is provably helpful in reducing incurred regret in bandit algorithms: there are well-known instance-dependent regret bounds for bandit algorithms which inversely relate the expected regret incurred by an algorithm to the success probability gaps between the optimal and suboptimal arms, with smaller gaps requiring more exploration to identify the best arm.
For instance, given a bandit instance with arm success probabilities $p_1 \geq \ldots \geq p_k$, it is known that UCB achieves expected instance-dependent regret of $O(\log (T) \cdot \sum_{i=2}^k \frac{1}{p_1 - p_i})$ \cite{bubeck2012regret,slivkins2019introduction,lattimore2020bandit}.
Thus, collapsing stages as shown above will provably reduce our incurred regret, as supported by our empirical experiments in \cref{sec:experiments}.

We turn this intuition into the \textsc{StagedCollapsedBandit} algorithm, which creates one $\textsc{Bandit}_j$ instance for each stage type $j$. 
In each round, the algorithm queries $\textsc{Bandit}_j$ once and chooses the returned action for each stage of type $j$.
If we pass \emph{all} stages of some type $j$, we give a reward of $1$ to $\textsc{Bandit}_j$.
If we fail for the first time at some stage of type $j$, we give a reward of $0$ to $\textsc{Bandit}_j$. 
Thus, each $\textsc{Bandit}_j$ is effectively acting on the product of success probabilities.

We also present a second adaptation of \textsc{StagedBandit} called \textsc{StagedCollapsedFineGrainedBandit}. 
Again, we create one $\textsc{Bandit}_j$ instance for each stage type $j$.
However now, for each stage, we query the corresponding $\textsc{Bandit}$ and feed back whether the selected action was successful in the stage. 

In \cref{sec:experiments}, we empirically show that neither algorithm dominates the other.
To understand this behavior in theory, let $f: [m] \to [\ell]$ be a mapping that maps each stage to one of $\ell$ types and let us consider the special case where all stages of the same type have \emph{identical} success probabilities.
Let us denote the action probabilities of stage type $j \in [\ell]$ by $\bQ_j = (q_1, \ldots, q_k)$ and its $r^{th}$ power by $\bQ^r_j = (q^r_1, \ldots, q^r_k)$.
Now, suppose we use the UCB algorithm for the $\textsc{Bandit}$s in both algorithms, where we again write $O(R_{T,{p_1,\ldots,p_k}})$ to denote the UCB regret.
Our analysis in \cref{sec:appendix-known-stage-types} implies that \textsc{StagedCollapsedBandit} would incur an expected regret of at most $O(\sum_{j=1}^\ell R_{T, \bQ^r_j})$ while \textsc{StagedCollapsedFineGrainedBandit} incurs an expected regret of at most $O(\sum_{j=1}^\ell R_{T \cdot |f^{-1}(j)|, \bQ_j})$.
Recalling that $O(R_{T,{p_1,\ldots,p_k}})\in O(\log (T) \cdot \sum_{i=2}^k \frac{1}{p_1 - p_i})$, these bounds suggest that \textsc{StagedCollapsedFineGrainedBandit} will perform advantageous when there are large differences between the success probabilities of actions for an individual stage, whereas \textsc{StagedCollapsedBandit} is favorable when multiplying  success probabilities of an action across stages of the same type leads to increased gaps. 

\section{Empirical Evaluation}
\label{sec:experiments}

While \textsc{StagedBandit} (\cref{alg:stagedbandit}) can be used for all SFIPP instances, we show in the experiments that depending on the structure of an SFIPP instance, our customized algorithms perform significantly better.
Thus, our proposed algorithms together form a toolbox that produces high-quality solutions applicable to a wide range of scenarios.

In the rest of this section, we will refer to our algorithm
\textsc{UniformThenFixed} as UTF, \textsc{StagedBandit} as SB, \textsc{StagedCollapsedBandit} as SCB, and \textsc{StagedCollapsedFineGrainedBandit} as SCFGB.
For experiments involving SCB and SCFGB, we either treat all stages as a single stage type (SCB\_1 and SCFBG\_1) or every alternate stage as the same stage type (SCB\_2 and SCFBG\_2).
Finally, in our implementations of SB, SCB, and SCFBG, we use the classic \textsc{UCB} algorithm \cite{auer2002finite} as our single-stage bandit.
For reproducibility, our source code and experimental scripts are provided in the supplementary materials.

\begin{figure*}[htb]
\centering
\begin{subfigure}[t]{0.32\linewidth}
    \centering
    \includegraphics[width=\linewidth]{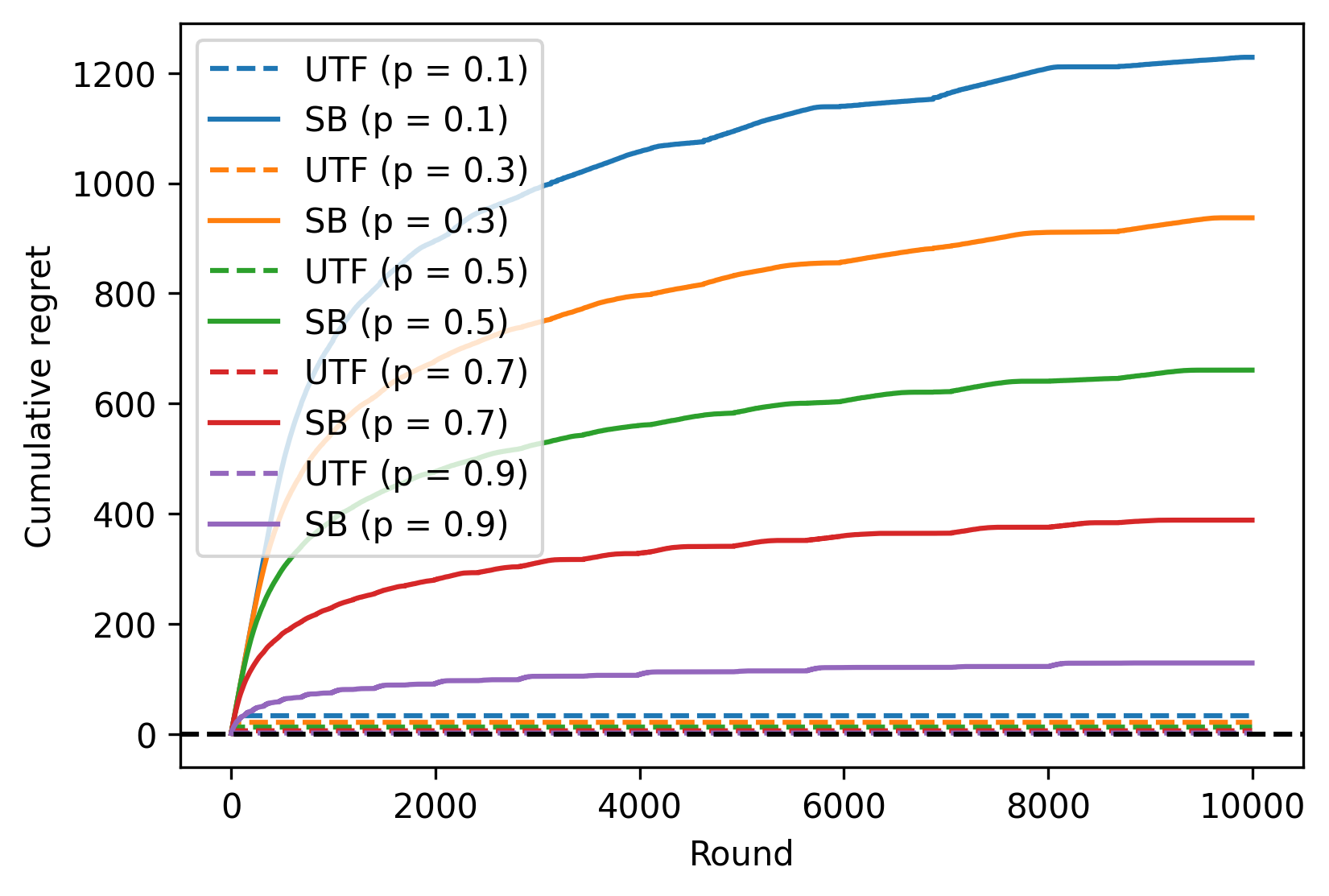}
    \caption{Deterministic processes}
\end{subfigure}
\hfill
\begin{subfigure}[t]{0.32\linewidth}
    \centering
    \includegraphics[width=\linewidth]{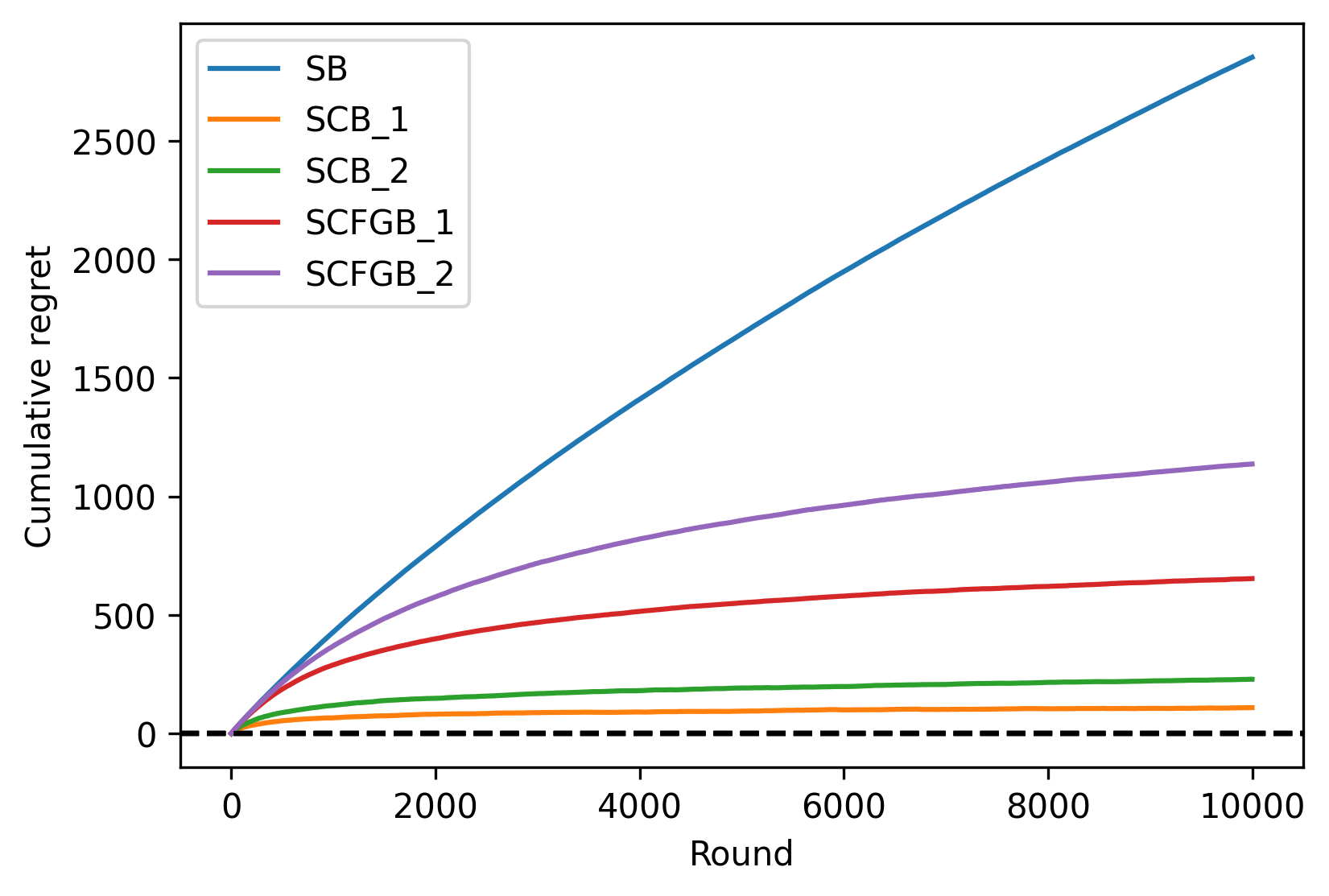}
    \caption{Probabilistic processes generated using $\mathrm{Beta}(\alpha = 10, \beta = 1)$  with one stage type.}
\end{subfigure}
\hfill
\begin{subfigure}[t]{0.32\linewidth}
    \centering
    \includegraphics[width=\linewidth]{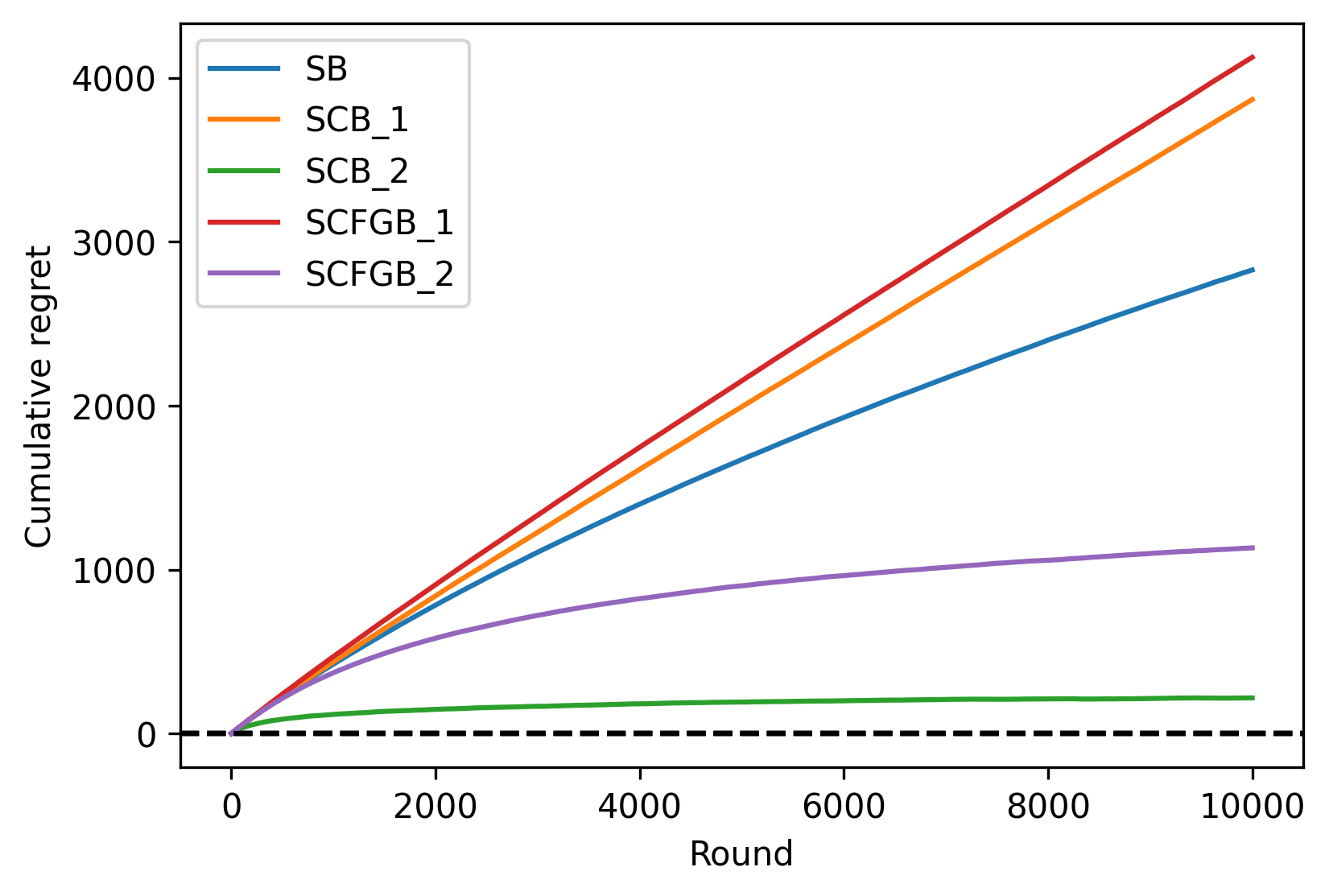}
    \caption{Probabilistic processes generated using $\mathrm{Beta}(\alpha = 10, \beta = 1)$ with two different alternating stage types.}
\end{subfigure}
\vspace{10pt}
\begin{subfigure}[t]{0.32\linewidth}
    \centering
    \includegraphics[width=\linewidth]{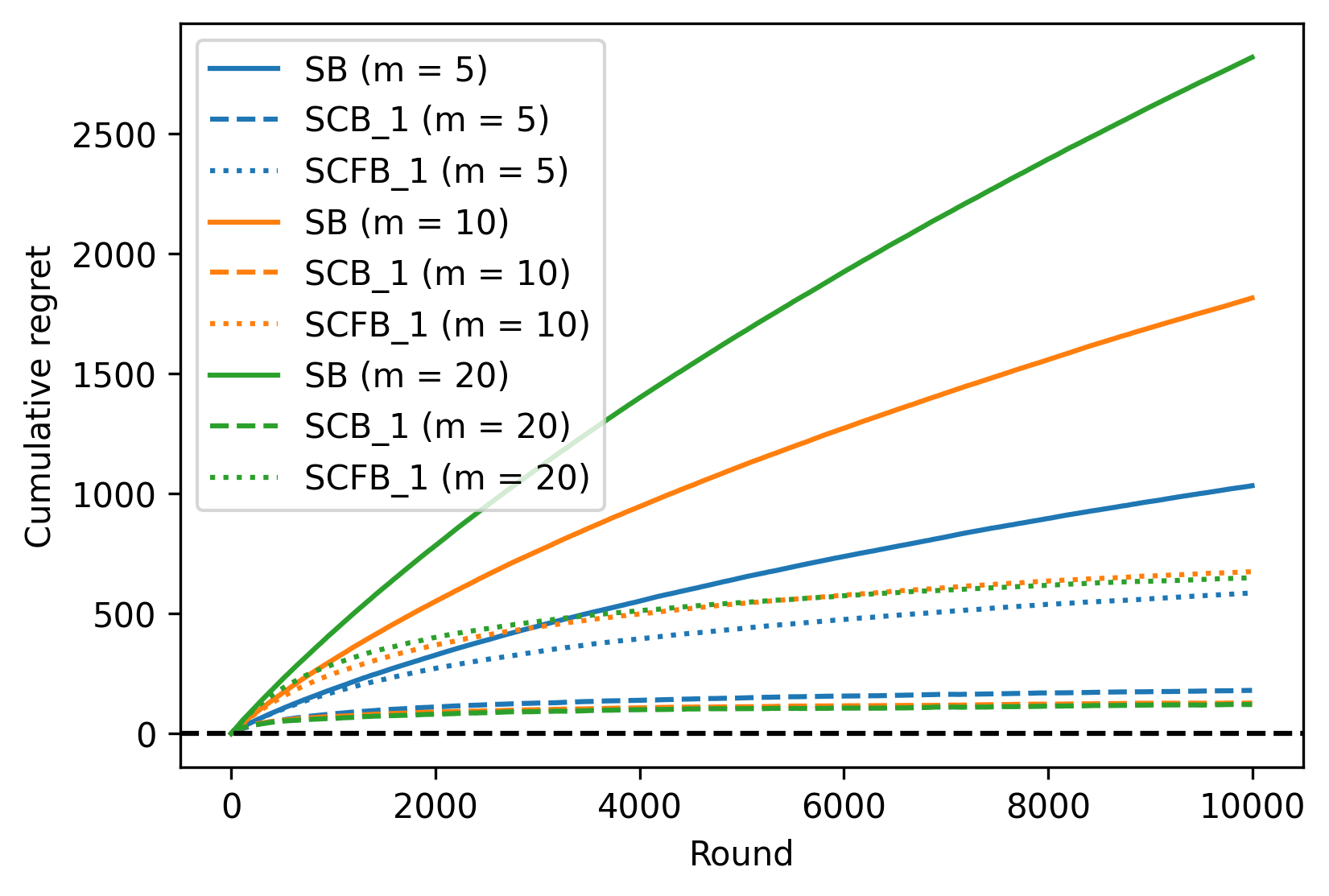}
    \caption{Probabilistic processes generated using
     $\mathrm{Beta}(\alpha = 10, \beta = 1)$ with one stage type, but with varying number of stages.}
\end{subfigure}
\hfill
\begin{subfigure}[t]{0.32\linewidth}
    \centering
    \includegraphics[width=\linewidth]{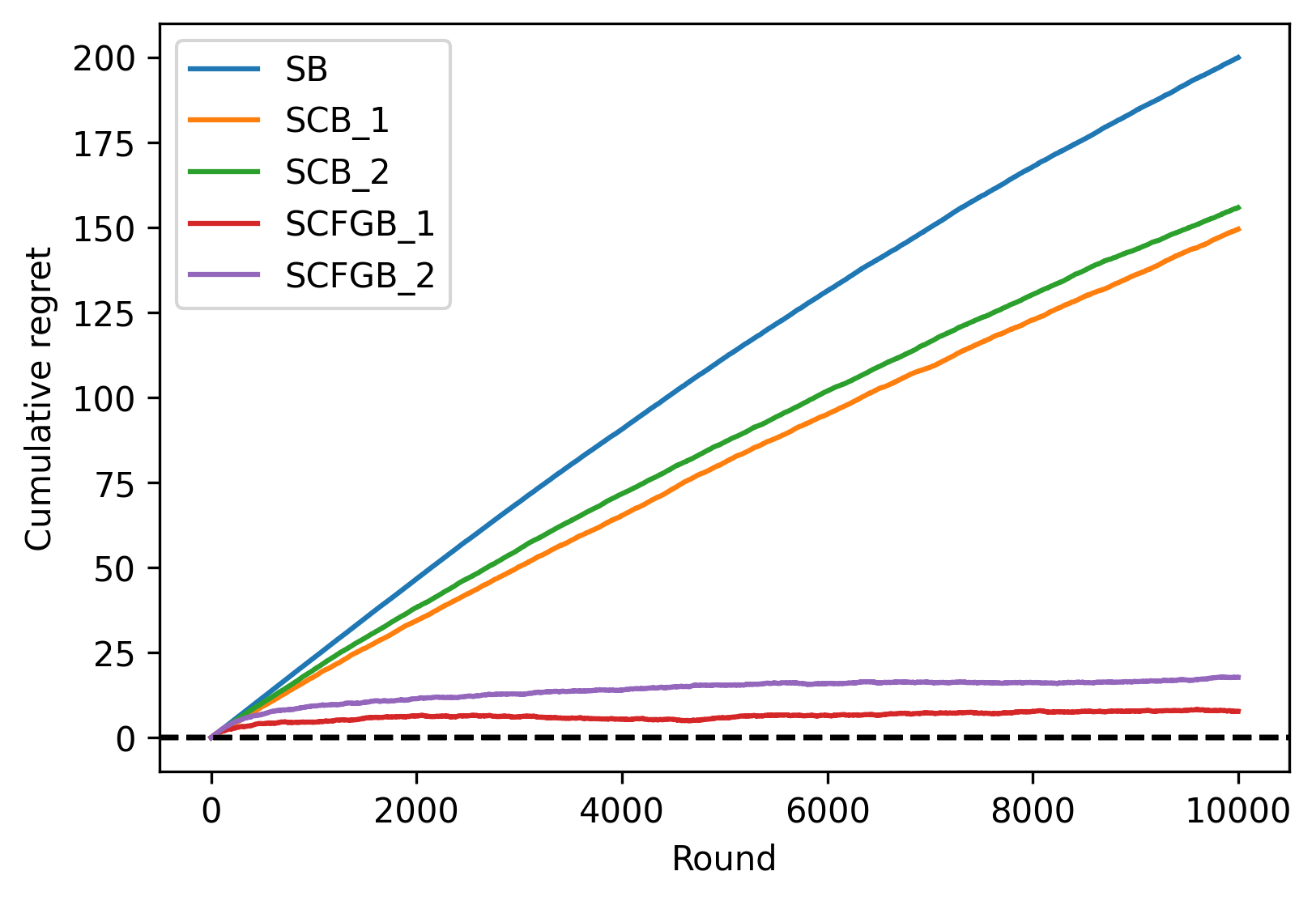}
    \caption{Probabilistic processes generated using $\mathrm{Beta}(\alpha = 1, \beta = 1)$ with one stage type.}
\end{subfigure}
\hfill
\begin{subfigure}[t]{0.32\linewidth}
    \centering
    \includegraphics[width=\linewidth]{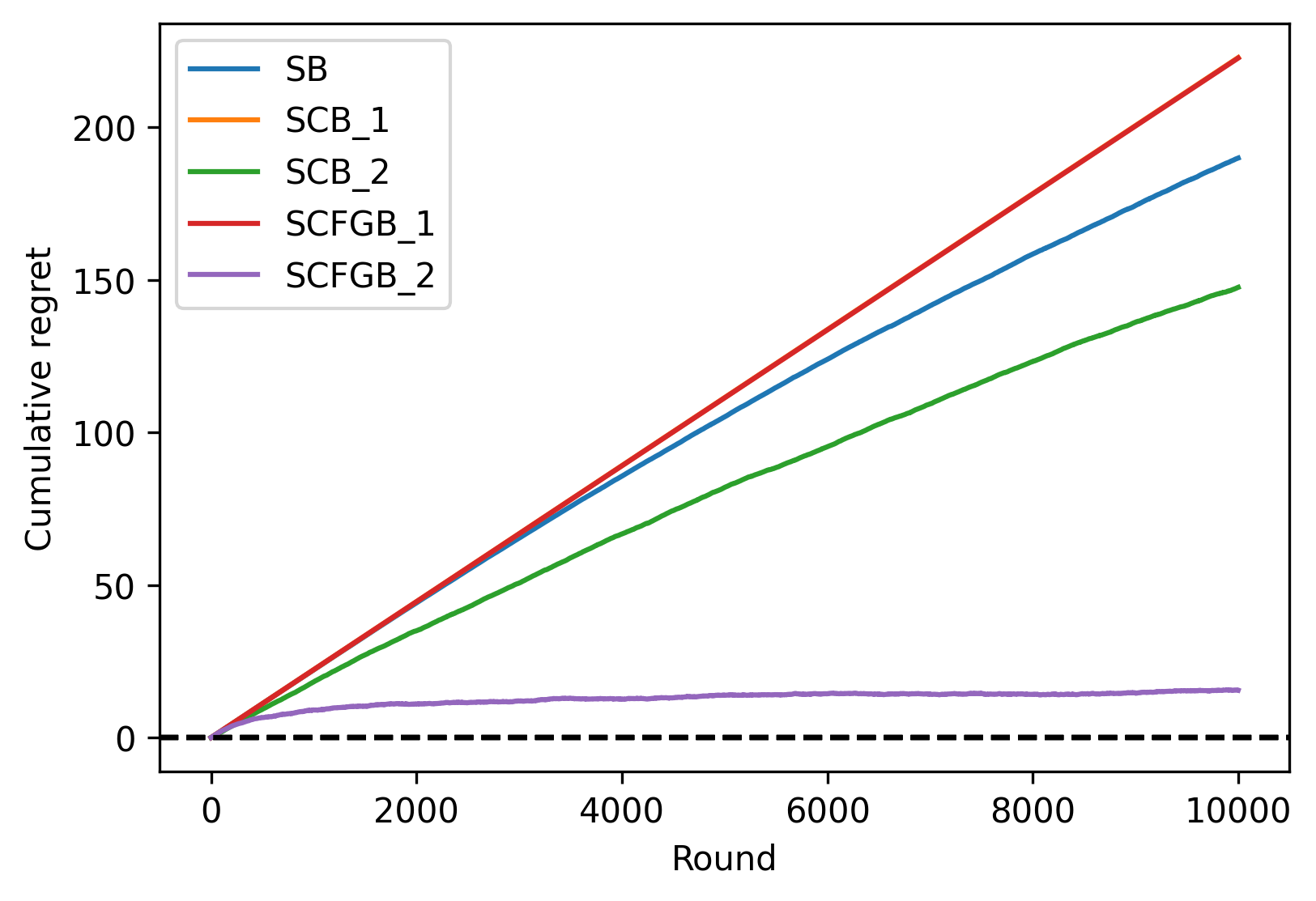}
    \caption{Probabilistic processes generated using $\mathrm{Beta}(\alpha = 1, \beta = 1)$
    with two different alternating stage types}
\end{subfigure}
\vspace{-0.5cm}
\caption{All experimental plots for \cref{sec:experiments}.
We provide error bar versions of them in \cref{sec:appendix-plots}.
}
\label{fig:all-experimental-plots}
\end{figure*}

\subsection{Power of Simplicity for Deterministic SFIPP}
We first look at the deterministic setting, where we will observe that our specialized algorithm UTF outperforms the more general SB algorithm.
For this, we constructed $100$ deterministic SFIPP instances with $m=20$ stages and $k=5$ actions across a horizon of $T = 10,000$.
We independently set each entry of $\bP$ to $\Pone$ with some fixed probability $p$.
We set one randomly selected index per stage to $\Pone$, thus ensuring that each stage contains at least one successful action.

\cref{fig:all-experimental-plots}(a) depicts the
average regret we obtained for the two algorithms and for~$p \in \{0.1, 0.3, 0.5, 0.7, 0.9\}$.
Each pair of colored curves compares the cumulative regret between UTF and SCB for some value of $p$.
The large gaps between same-colored curves indicate the decisive superiority of UTF due
to excessive exploration of unsuccessful actions by the general bandit approach SCB, as it does not incorporate that the setting is deterministic.

\subsection{More Stages Brings More Collapsing Gains}
\label{sec:valid-collapsing}

We empirically validate our theory from \cref{sec:known-stage-types} that ``collapsing'' multiple stages of the same time helps drastically as the number of stages $m$ increases.
We generated $100$ probabilistic SFIPP instances with $m=\{5,10,20\}$ stages and $k=5$ actions.
Each entry of $\bP$ is generated i.i.d.\ using the Beta distribution with parameters $\alpha = 10$ and $\beta = 1$ resulting in large success probabilities.
By relabeling the actions, we ensured that in each generated instance, the same single action is optimal for \emph{every} stage; see \cref{sec:appendix-staged-generation} for details.

We show the averaged regret values incurred by SB, SCB\_1, and SCFGB\_1 for a horizon of $T = 10,000$ in \cref{fig:all-experimental-plots}(b).
The curves for SCB\_1 and SCFGB\_1 flatten out much quicker as compared to SB, for all values of $m$.
This is because SCB\_1 and SCFGB\_1 can quickly learn to perform well in later stages by observing the failures and successes incurred at earlier stages, while SB essentially treats all stages as fully independent, even though the optimal action index is the same across all stages.

\subsection{Any Valid Collapsing Helps Reduce Regret}

We now take a closer look at the comparison between our two customized algorithms SCB and SCFGB for the case with known stage types.
We prepared $100$ probabilistic SFIPP instances with $m=20$ stages and $k=5$ actions across a horizon of $T = 10,000$.
Each entry of $\bP$ is generated i.i.d.\ using the Beta distribution, either with parameters $\alpha = 10$ and $\beta = 1$ to encourage large success probabilities or with parameters $\alpha = 1$ and $\beta = 1$ to sample probability values from the uniform distribution.
By relabeling the actions if necessary, we ensured that in each generated instance, the same single action is optimal for \emph{every} stage; see \cref{sec:appendix-staged-generation} for details.

\cref{fig:all-experimental-plots}(c,d) illustrate the average results for three algorithms SB, SCB, and SCFGB on these instances; (c) shows results for instances generated with parameters $\alpha = 10$ and $\beta = 1$, while (d) shows results for those generated with parameters $\alpha = 1$ and $\beta = 1$.
It is clear that both collapsed bandits SCB and SCFGB obtain much lower expected regret compared to SB.
There are two further interesting observations to point out here.
Firstly, despite not fully collapsing all stages into a single stage type, the curves of SCB\_2 and SCFGB\_2 show that a valid partial collapsing into two stage types still helps to reduce expected regret as compared to SB.
Secondly, as discussed in \cref{sec:known-stage-types}, neither SCB nor SCFGB always dominates the other: SCB beats SCFGB in instances generated with Beta distribution parameters $\alpha = 10$ and $\beta = 1$ (which generally leads to instances with high success probabilities where differences of products of success probabilities are larger than the differences in the original probabilities) while SCFGB beats SCB in instances generated with Beta distribution parameters $\alpha = 1$ and $\beta = 1$ (where we expect larger gaps in the success probabilities at each stage).

\subsection{Invalid Collapsing Increases Regret}

We demonstrate a strong detrimental effect of unbounded regret
resulting from collapsing stages that are
of different type.
To illustrate this effect on an example first, suppose we collapse the last two rows of $\bP$ from the beginning of \Cref{sec:known-stage-types} into $\bP''$:
\[
\bP''
=
\begin{bmatrix}
0.9 & 0.8\\
0.6 \cdot 0.5 & 0.5 \cdot 0.7
\end{bmatrix}
=
\begin{bmatrix}
0.9 & 0.8\\
0.3 & 0.35
\end{bmatrix}.
\]
In the collapsed matrix $\bP''$, the optimal action sequence $(1,2)$ corresponds to performing $(1,2,2)$ in $\bP$, which is suboptimal and will always incur regret even as $T \to \infty$.

Turning to our experiments, as before, we construct $100$ probabilistic SFIPP instances with $m=20$ stages and $k=5$ actions across a horizon of $T = 10,000$.
Each entry of $\bP$ is generated i.i.d.\ using the Beta distribution, either with parameters $\alpha = 10$ and $\beta = 1$ to encourage large success probabilities or with parameters $\alpha = 1$ and $\beta = 1$ to sample probability values from the uniform distribution.
By relabeling the actions if necessary, we ensured that in each generated instance, odd-numbered and even-numbered stages are of the same stage type, i.e., all odd-numbered stages share a common optimal action, as do all even-numbered stages, but the optimal actions for odd and even stages are distinct; see \cref{sec:appendix-staged-generation} for details.

In \cref{fig:all-experimental-plots}(e,f), we see that SCB\_1 and SCFGB\_1 incur ever-increasing regret due to misalignment of stage type information; (e) shows results for instances generated with parameters $\alpha = 10$ and $\beta = 1$ while (f) shows results for those generated with parameters $\alpha = 1$ and $\beta = 1$.
Since they combined both odd and even type stages into a single one, SCB\_1 and SCFGB\_1 would \emph{never} be able to learn and choose the true optimal action across all $m$ stages.
Meanwhile, we see that SCB\_2 and SCFGB\_2 outperform SB as expected, due to a similar reasoning as in \cref{sec:valid-collapsing}.

\section{Conclusions}
\label{sec:conclusion}

In this work, we formalized a novel sequential planning problem called SFIPP and proposed algorithms for solving it under various settings.
We supplemented our theoretical guarantees with empirical evaluation, highlighting interesting trade-offs between the algorithms depending on the \emph{unknown} underlying success probabilities.

In future work, it would be interesting to consider different generalizations of our new model, for instance, the case where a different set of actions is available at different timesteps or stages, thereby arriving at a generalization of sleeping multi-armed bandits \cite{kleinberg2010regret}.

\section*{Acknowledgements}
This work was supported by the Office of Naval Research (ONR) under Grant Number N00014-23-1-2802.
The views and conclusions contained in this document are those of the authors and should not be interpreted
as necessarily representing the official policies, either expressed or implied, of the Office of Naval Research
or the U.S. Government. AK acknowledges support, in part, from the European Research Council (ERC) under the European Union’s Horizon 2020 research and innovation programme (grant agreement No 101002854).

\begin{figure}[H]
\centering
\includegraphics[width=3cm]{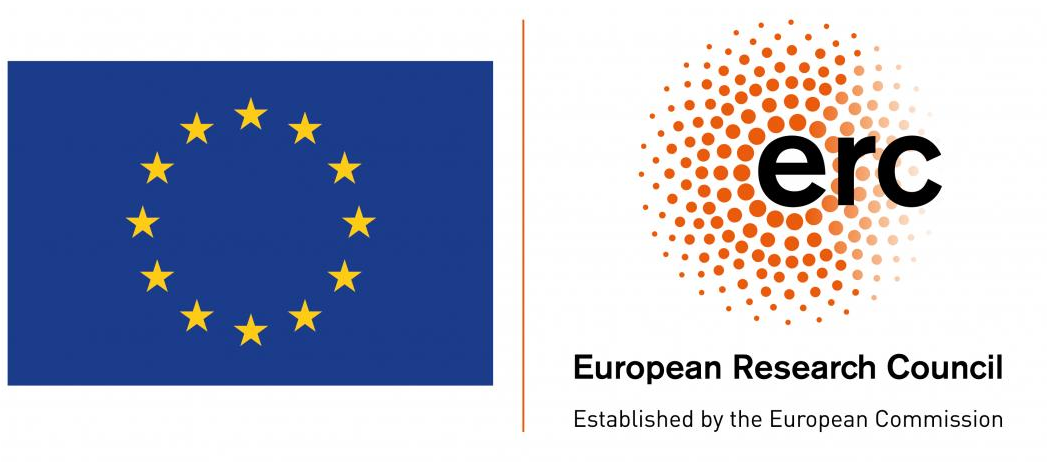}
\end{figure}

\bibliographystyle{alpha}
\bibliography{ref}

\appendix

\section{Details for Improving Regret with Known Stage Types}
\label{sec:appendix-known-stage-types}

In this section, we give the pseudocode of our two collapsed bandit algorithms \textsc{StagedCollapsedBandit} (\cref{alg:stagedcollapsedbandit}) and \textsc{StagedCollapsedFineGrainedBandit} (\cref{alg:stagedcollapsedfinegrainedbandit}), and their accompanying theoretical guarantees in \cref{thm:staged-collapsed-bandit-guarantees} and \cref{thm:staged-collapsed-fine-grained-bandit-guarantees} respectively.

\begin{algorithm}[htb]
\begin{algorithmic}[1]
\caption{The \textsc{StagedCollapsedBandit} algorithm.}
\label{alg:stagedcollapsedbandit}
    \Statex \textbf{Input}: Number of rounds $T$, number of stages $m$, number of actions $k$, a classic single-stage bandit algorithm \textsc{Bandit}, stage type function $f: [m] \to [\ell]$
    \State Let $\ell$ be the number of stage types, i.e.\ $\ell = \max_{s \in m} f(s)$
    \State Initialize $\ell$ \textsc{Bandit} instances $\textsc{Bandit}_1$, $\ldots$, $\textsc{Bandit}_\ell$, one for each stage type
    \For{round $t = 1, \ldots, T$}
        \For{stage $s = 1, \ldots, m$}
            \If{this is the first instance of type $f(s)$ stage}
                \State Query $\textsc{Bandit}_{f(s)}$ for an action $i_{f(s)} \in [k]$
            \EndIf
            \State Use action $i_{f(s)} \in [k]$ for stage $s$ in round $t$
            \If{stage $s$ succeeded and $s = m$}
                \State Provide all bandits $\textsc{Bandit}_1, \ldots, \textsc{Bandit}_\ell$ with a reward of 1
            \Else \Comment{stage $s$ failed or $s < m$}
                \If{stage $s$ failed}
                    \State Provide $\textsc{Bandit}_{f(s)}$ with a reward of 0
                    \State Provide reward of 1 to any \emph{other} bandits that have completed all stages
                \EndIf
                \State\hspace{\algorithmicindent}\textbf{break}
                \Comment{Do not perform future stages}
            \EndIf
        \EndFor
    \EndFor
\end{algorithmic}
\end{algorithm}

\begin{algorithm}[htb]
\begin{algorithmic}[1]
\caption{The \textsc{StagedCollapsedFineGrainedBandit} algorithm.}
\label{alg:stagedcollapsedfinegrainedbandit}
    \Statex \textbf{Input}: Number of rounds $T$, number of stages $m$, number of actions $k$, a classic single-stage bandit algorithm \textsc{Bandit}, stage type function $f: [m] \to [\ell]$
    \State Let $\ell$ be the number of stage types, i.e.\ $\ell = \max_{s \in m} f(s)$
    \State Initialize $\ell$ \textsc{Bandit} instances $\textsc{Bandit}_1$, $\ldots$, $\textsc{Bandit}_\ell$, one for each stage type
    \For{round $t = 1, \ldots, T$}
        \For{stage $s = 1, \ldots, m$}
            \State Query $\textsc{Bandit}_{f(s)}$ for an action $i_{f(s)} \in [k]$
            \State Use action $i_{f(s)} \in [k]$ for stage $s$ in round $t$
            \If{stage $s$ succeeded}
                \State Provide $\textsc{Bandit}_{f(s)}$ with a reward of $1$
            \Else
                \State Provide $\textsc{Bandit}_{f(s)}$ with a reward of $0$
                \State \textbf{break}
                \Comment{Do not perform stages $s+1, \ldots, m$}
            \EndIf
        \EndFor
    \EndFor
\end{algorithmic}
\end{algorithm}

The analysis of \textsc{StagedCollapsedFineGrainedBandit} relies \cref{lem:prod-to-sum} while the analysis of \textsc{StagedCollapsedBandit} relies on a generalization of \cref{lem:prod-to-sum} where we also account for the stage types.

\begin{restatable}{lemma}{prodtosumtypes}
\label{lem:prod-to-sum-types}
Let $a_1, \ldots, a_m, b_1, \ldots, b_m \in [0,1]$ such that $0 \leq b_i \leq a_i \leq 1$ for all $i \in [m]$.
Suppose there is a function $f: [m] \to [\ell]$ that maps stages into one of $\ell$ types.
Then,
\[
0
\leq \left( \prod_{i=1}^m a_i \right) - \left( \prod_{i=1}^m b_i \right)\\
\leq \sum_{j=1}^\ell \left( \left( \prod_{i \in [m] : f(i) = j} a_i \right) - \left( \prod_{i \in [m] : f(i) = j} b_i \right) \right)
\]
\end{restatable}
\begin{proof}
We perform induction over $\ell$.
If $\ell = 1$, then there is only one stage type and the claim holds trivially.
Now suppose $\ell > 1$.
The inequality $0 \leq \left( \prod_{i=1}^m a_i \right) - \left( \prod_{i=1}^m b_i \right)$ follows from $0 \leq b_i \leq a_i$ for all $i \in [m]$.
Meanwhile, by relabeling the indices, suppose that $f(i) \in [\ell - 1]$ for $1 \leq i \leq k$ and $f(i) = \ell$ for $k+1 \leq i \leq m$, for some $k \in [m]$.
\begin{align*}
&\; \left( \prod_{i=1}^m a_i \right) - \left( \prod_{i=1}^m b_i \right)\\
= &\; \left( \left( \prod_{i=1}^{k} a_i \right) - \left( \prod_{i=1}^{k} b_i \right) \right) \cdot \left( \prod_{i=k+1}^{m} a_i \right) + \left( \prod_{i=1}^{k} b_i \right) \cdot \left( \left( \prod_{i=k+1}^{m} a_i \right) - \left( \prod_{i=k+1}^{m} b_i \right) \right)\\
\leq &\; \left( \left( \prod_{i=1}^{k} a_i \right) - \left( \prod_{i=1}^{k} b_i \right) \right) + \left( \left( \prod_{i=k+1}^{m} a_i \right) - \left( \prod_{i=k+1}^{m} b_i \right) \right) \tag{Since $\prod_{i=k+1}^{m} a_i \leq 1$ and $\prod_{i=1}^{k} b_i \leq 1$}\\
\leq &\; \sum_{j=1}^{\ell-1} \left( \left( \prod_{i \in [m] : f(i) = j} a_i \right) - \left( \prod_{i \in [m] : f(i) = j} b_i \right) \right) + \left( \left( \prod_{i=k+1}^{m} a_i \right) - \left( \prod_{i=k+1}^{m} b_i \right) \right) \tag{By induction hypothesis and $f(i) \in [m]$ for $1 \leq i \leq k$}\\
= &\; \sum_{j=1}^{\ell} \left( \left( \prod_{i \in [m] : f(i) = j} a_i \right) - \left( \prod_{i \in [m] : f(i) = j} b_i \right) \right) \tag{Since $f(i) = \ell$ for $k+1 \leq i \leq m$}
\end{align*}
\end{proof}

\begin{theorem}[Guarantees of \textsc{StagedCollapsedBandit}]
\label{thm:staged-collapsed-bandit-guarantees}
Let $\bP \in \R^{m \times k}$ be a SFIPP instance matrix with rows $\bP_1, \ldots, \bP_m$.
Suppose there is a single-stage bandit algorithm that achieves regret bound of $O(R_{T, p_1, \ldots, p_k})$ on the set of Bernoulli arms with success probabilities $p_1, \ldots, p_k$ over a horizon of $T$ rounds.
Let $f: [m] \to [\ell]$ be the stage type function mapping stage indices $s \in [m]$ to a stage type $j \in [\ell]$.
Then, there is an algorithm that achieving regret bound of $O(\sum_{j=1}^\ell R_{T,\bP'_j})$, where $\bP'_j = ( \prod_{i \in [m]: f(i) = j} p_{i,1}, \ldots, \prod_{i \in [m]: f(i) = j} p_{i,k}) \in [0,1]^k$.
\end{theorem}
\begin{proof}
Consider the \textsc{StagedCollapsedBandit} algorithm (\cref{alg:stagedcollapsedbandit}) which runs $\ell$ instances of a single-stage bandit algorithm \textsc{Bandit}.
Since \textsc{StagedCollapsedBandit} only chooses an action once for each stage type within each round and provide a positive bandit feedback when \emph{all} stages succeed, the $j^{th}$ instance $\textsc{Bandit}_j$ is effectively performing a bandit instance with probabilities $\bP'_j = ( \prod_{i \in [m]: f(i) = j} p_{i,1}, \ldots, \prod_{i \in [m]: f(i) = j} p_{i,k}) \in [0,1]^k$ for any $f(s) = j$ and thus achieves an instance-dependent regret of $O(R_{T,\bP'_j})$ over $T$ rounds.
\begin{align*}
&\; \mathbb{E} \left[ R(\textsc{StagedCollapsedBandit}) \right]\\
= &\; \mathbb{E} \left[ \sum_{t=1}^T \left[ \left( \prod_{s=1}^m p_{s,i^*_s} \right) - \left( \prod_{s=1}^m p_{s,i_{t,s}} \right) \right] \right] \tag{By definition of \cref{eq:SFIPP-benchmark-objective}}\\
\leq &\; \mathbb{E} \left[ \sum_{t=1}^T \sum_{j=1}^\ell \left( \left( \prod_{i \in [m]: f(i) = j} p_{j,i^*_j} \right) - \left( \prod_{i \in [m]: f(i) = j} p_{j,i_{t,j}} \right) \right) \right] \tag{By \cref{lem:prod-to-sum-types} since $p_{j,i^*_j} \geq p_{j,i_{t,j}}$ always}\\
= &\; \sum_{j=1}^\ell \mathbb{E} \left[ \sum_{t=1}^T \left( \left( \prod_{i \in [m]: f(i) = j} p_{j,i^*_j} \right) - \left( \prod_{i \in [m]: f(i) = j} p_{j,i_{t,j}} \right) \right) \right] \tag{Linearity of expectation}\\
\in &\; O(\sum_{j=1}^\ell R_{T,\bP'_j}) \tag{Since each $\textsc{Bandit}_j$ incurs a regret of $O(R_{T,\bP'_j})$}
\end{align*}
\end{proof}

\begin{theorem}[Guarantees of \textsc{StagedCollapsedFineGrainedBandit}]
\label{thm:staged-collapsed-fine-grained-bandit-guarantees}
Suppose there is a single-stage bandit algorithm that achieves regret bound of $O(R_{T, p_1, \ldots, p_k})$ on the set of Bernoulli arms with success probabilities $p_1, \ldots, p_k$ over a horizon of $T$ rounds.
Let $f: [m] \to [\ell]$ be the stage type function mapping stage indices $s \in [m]$ to a stage type $j \in [\ell]$.
Under the assumption that all stages of the same type are \emph{identical}, let $\bQ_j = (q_{j,1}, \ldots, q_{j,k})$ denote the action probabilities of stage type $j \in [\ell]$.
Then, there is an algorithm that achieving a regret bound of 
$O(\sum_{j=1}^\ell R_{T \cdot |f^{-1}(j)|, \bQ_{j}})$. 
\end{theorem}
\begin{proof}
Let us denote the optimal action in stages of type $j$ by $l^*_j \in \argmax_{i\in [k]} q_{j,i}$.
Consider the \textsc{StagedCollapsedFineGrainedBandit} algorithm (\cref{alg:stagedcollapsedfinegrainedbandit}).
For $j \in [\ell]$, $t\in [T]$, and $b \in [|f^{-1}(\ell)|]$, let $l_{t,j,b}$ be the action returned by $\textsc{Bandit}_j$ for the $b$th stage of type $j$ in round $t$. 

\begin{align*}
&\; \mathbb{E} \left[ R(\textsc{StagedCollapsedFineGrainedBandit}) \right]\\
= &\; \mathbb{E} \left[ \sum_{t=1}^T \left[ \left( \prod_{s=1}^m p_{s,i^*_s} \right) - \left( \prod_{s=1}^m p_{s,i_{t,s}} \right) \right] \right] \tag{By definition of \cref{eq:SFIPP-benchmark-objective}}\\
\leq &\; \mathbb{E} \left[ \sum_{t=1}^T \sum_{s=1}^m \left( p_{s,i^*_s} - p_{s,i_{t,s}} \right) \right] \tag{By \cref{lem:prod-to-sum} since $p_{s,i^*_s} \geq p_{s,i_{t,s}}$ always}\\
= &\; \mathbb{E} \left[ \sum_{t=1}^T \sum_{j=1}^\ell \sum_{s\in f^{-1}(j)} \left( p_{s,i^*_s} - p_{s,i_{t,s}} \right) \right] \\
= &\; \sum_{j=1}^\ell \mathbb{E} \left[ \sum_{t=1}^T  \sum_{s\in f^{-1}(j)} \left( p_{s,i^*_s} - p_{s,i_{t,s}} \right) \right] \tag{Linearity of expectation}\\
= &\; \sum_{j=1}^\ell \mathbb{E} \left[ \sum_{t=1}^T  \sum_{s\in f^{-1}(j)} \left( q_{j,l^*_{j}} - q_{j,i_{t,s}} \right) \right] \\
= &\; \sum_{j=1}^\ell \mathbb{E} \left[ \sum_{t=1}^T  \sum_{b\in [|f^{-1}(j)|]} \left( q_{j,l^*_{j}} - q_{j,l_{t,j,b}} \right) \right] \\
\in &\; O(\sum_{j=1}^\ell R_{|f^{-1}(j)|T,\mathbf{Q}_j})
\end{align*}
The final inclusion step is because, for any $j \in [\ell]$, $\textsc{Bandit}_j$ incurs a regret of $O(R_{|f^{-1}(j)|T,\bP_s})$ when queried $T \cdot |f^{-1}(j)|$ times.
\end{proof}

In the following, we provide an example illustrating that the theoretical analysis of neither algorithm dominates the other; one may be better than the other depending on the actual distribution of probabilities in the underlying $\bP$.
Consider the special case where all stages of the same type are \emph{identical}.
That is, given SFIPP instance matrix $\bP \in \R^{m \times k}$ with rows $\bP_1, \ldots, \bP_m$, $f(i) = f(i')$ implies that $\bP_i = \bP_{i'}$.
Henceforth, let us denote the action probabilities of stage type $j \in [\ell]$ by $\bQ_j = (q_1, \ldots, q_k)$ and its $r^{th}$ power by $\bQ^r_j = (q^r_1, \ldots, q^r_k)$.

Now, suppose we use the UCB algorithm for the $\textsc{Bandit}$s in both algorithms, where we again write $O(R_{T,{p_1,\ldots,p_k}})$ to denote the regret the UCB algorithm encounters on a set of Bernoulli arms with success probabilities $p_1\geq \dots\geq p_k$ over a horizon of $T$ rounds.
In our analysis in \cref{sec:appendix-known-stage-types}, we show that \textsc{StagedCollapsedBandit} would incur an expected regret of at most $O(\sum_{j=1}^\ell R_{T, \bQ^r_j})$ while \textsc{StagedCollapsedFineGrainedBandit} incurs an expected regret of at most $O(\sum_{j=1}^\ell R_{T \cdot |f^{-1}(j)|, \bQ_j})$.
Recalling that $O(R_{T,{p_1,\ldots,p_k}})\in O(\log (T) \cdot \sum_{i=2}^k \frac{1}{p_1 - p_i})$, these bounds suggest that \textsc{StagedCollapsedFineGrainedBandit} will perform advantageous when there are large differences between the success probabilities of actions for an individual stage, whereas \textsc{StagedCollapsedBandit} is favorable when multiplying  success probabilities of an action across stages of the same type leads to increased gaps.
The following example makes this intuition more concrete.

\begin{example}
Suppose $m = 10$ and $\ell = 2$ over a horizon of $T = 10$ with SFIPP matrices $\bP$ and $\bP'$, where the first five and last five stages are identical in both $\bP$ and $\bP'$:
\[
\bP =
\begin{bmatrix}
0.9 & 0.8\\
\vdots & \vdots\\
0.9 & 0.8\\
0.7 & 0.8\\
\vdots & \vdots\\
0.7 & 0.8
\end{bmatrix}
\qquad
\bP' =
\begin{bmatrix}
0.9 & 0.1\\
\vdots & \vdots\\
0.9 & 0.1\\
0.2 & 0.3\\
\vdots & \vdots\\
0.2 & 0.3\\
\end{bmatrix}
\]
With respect to $\bP$, UCB has an instance-dependent regret bound of $O(\frac{\log(10)}{0.9 - 0.8})$ on type $1$ stages and $O(\frac{\log(10)}{0.8 - 0.7})$ on type $2$ stages.
So, we expect \textsc{StagedCollapsedBandit} to outperform \textsc{StagedCollapsedFineGrainedBandit} since $\frac{\log(10)}{0.9^5 - 0.8^5} + \frac{\log(10)}{0.8^5 - 0.7^5} \approx 23 < 78 \approx \frac{\log(50)}{0.9 - 0.8} + \frac{\log(50)}{0.8 - 0.7}$.
Meanwhile, with respect to $\bP'$, UCB has an instance-dependent regret bound of $O(\frac{\log(10)}{0.9 - 0.1})$ on type $1$ stages and $O(\frac{\log(10)}{0.3 - 0.2})$ on type $2$ stages.
So, we expect \textsc{StagedCollapsedBandit} to outperform \textsc{StagedCollapsedFineGrainedBandit} since $\frac{\log(10)}{0.9^5 - 0.1^5} + \frac{\log(10)}{0.3^5 - 0.2^5} \approx 1095 > 44 \approx \frac{\log(50)}{0.9 - 0.1} + \frac{\log(50)}{0.3 - 0.2}$.
\end{example}

\section{Deferred proofs}
\label{sec:appendix-proofs}

In this section, we give the deferred proof details.

\onestagelowerbound*
\begin{proof}
Without loss of generality, by relabeling indices, we may assume that any deterministic algorithm queries the array indices in order from $1$ to $k$ when searching for the first index containing a $\Pone$.
Now, consider the uniform distribution of all $\binom{k}{z}$ binary arrays over $\bA_z$.
In general, for $0 \leq i \leq z$, there will be $\binom{k-i-1}{z-i}$ arrays starting with $i$ $\Pzero$s followed by a $\Pone$.
Each of these arrays require $(i+1)$ queries by the deterministic algorithm to locate the first index containing a $\Pone$.
So, in expectation over the uniform distribution, the deterministic algorithm will require
\[
\frac{1}{\binom{k}{z}} \sum_{i=0}^{z} (i+1) \cdot \binom{k-i-1}{z-i}
= \frac{k+1}{k+1-z}
\]
queries. Subtracting $1$ from the above, we get the expected number of
queries that the deterministic algorithm performs \emph{before} it 
queries the index with~$\Pone$.
By Yao's lemma \cite{yao1977probabilistic}, this means that any algorithm requires at least $\frac{z}{k+1-z}$ index queries 
before it identifies an index with $\Pone$ within arrays with exactly $z$ ones in the worst case.
\end{proof}

\onestageupperbound*
\begin{proof}
Let us define a random variable $Y$ counting the number of selections until, \emph{but excluding}, an $\Pone$ was selected.
We will show that $\E(Y)= \frac{1}{k+1-z}$.
To begin, let us number the $z$ $\Pzero$s in any order.
For each $i \in [z]$, we introduce an indicator random
variable $Z_i$, that is, $Z_i$ takes value $1$ if the $i$-th $\Pzero$ was selected before \emph{any} $\Pone$, or value $0$ otherwise.
Over the uniform distribution of all possible $k!$ sequences, $\E[Z_i]$ is exactly the probability that the $i^{th}$ $\Pzero$ occurs before all the $k-z$ $\Pone$s, which happens with proportion $\frac{1}{k+1-z}$, and so $\E(Z_i) = \frac{1}{k-z+1}$; see \cref{example:counting-Z_i} for an example of this counting argument.
Thus, by linearity of expectation, we see that
\[
\E(Y)
= \E \left(\sum_{i=1}^{z} Z_i \right)
= \sum_{i=1}^{z} \E(Z_i)
= \frac{z}{k+1-z}
\]
\end{proof}

\begin{example}
\label{example:counting-Z_i}
Suppose $k = 4$ and $z = 2$.
Let us fix an arbitrary identity to each $\Pzero$ and $\Pone$: $\{{\color{blue}\Pzero}, {\color{red}\Pzero}, {\color{green!70!black}\Pone}, {\color{orange}\Pone}\}$.
Now, let us count the number of times {\color{blue}$\Pzero$} appears before any occurrence of $\Pone$s amongst all possible length $k$ binary vectors with 2 zeroes:
\[
\begin{matrix}
{\color{red}\Pzero}{\color{blue}\Pzero}{\color{green!70!black}\Pone}{\color{orange}\Pone}
& {\color{blue}\Pzero}{\color{red}\Pzero}{\color{green!70!black}\Pone}{\color{orange}\Pone}
& {\color{blue}\Pzero}{\color{green!70!black}\Pone}{\color{red}\Pzero}{\color{orange}\Pone}
& {\color{blue}\Pzero}{\color{green!70!black}\Pone}{\color{orange}\Pone}{\color{red}\Pzero}\\
{\color{red}\Pzero}{\color{blue}\Pzero}{\color{orange}\Pone}{\color{green!70!black}\Pone}
& {\color{blue}\Pzero}{\color{red}\Pzero}{\color{orange}\Pone}{\color{green!70!black}\Pone}
& {\color{blue}\Pzero}{\color{orange}\Pone}{\color{red}\Pzero}{\color{green!70!black}\Pone}
& {\color{blue}\Pzero}{\color{orange}\Pone}{\color{green!70!black}\Pone}{\color{red}\Pzero}
\end{matrix}
\]
We see that there are 8 possibilities amongst $4! = 24$ the total of possibilities.
Since we are always uniformly selecting indices, this occurs with probability $\frac{1}{k+1-z} = \frac{1}{3} = \frac{8}{24!}$.
\end{example}

\hardestmultistagesetup*
\begin{proof}
Consider the following distribution of $z_1, \ldots, z_m$ into $m$ parts such that $z_1, \ldots, z_m = a(k+1) + b$: $\bz^* = (z^*_1, \ldots, z^*_m) = (\underbrace{k-1, \ldots, k-1}_{a}, b, 0, \ldots,0)$, which concentrates as many counts in the coordinates as possible, up to the constraint of $z_s \leq k+1$.
Such a distribution would correspond to the sum
\[
\sum_{s=1}^m \frac{z^*_s}{k+1-z^*_s}
= \sum_{s=1}^m \frac{z^*_s}{k+1-z^*_s}
= \frac{a}{2} + \frac{b}{k+1-b}
\]
Notice that this expression holds regardless of which index positions within $[m]$ are $0$ as long as the \emph{multiset} of counts across all $m$ indices contain $a$ counts of $(k-1)$, $1$ count of $r$, and the others being $0$.

Now, let us define the function $g(\bz) = g(z_1, \ldots, z_m) = \sum_{s=1}^m \frac{z_s}{k+1-z_s}$.
Since the function $g$ is symmetric in the input coordinates, we may assume without loss of generality that $z_1 \geq \ldots \geq z_m$ for any input $\bz = (z_1, \ldots, z_m)$ to $g$.
We will argue that $g(z_1, \ldots, z_m)$ is maximized when the counts are distributed according to $\bz^*$ via an exchange argument.

For any other vector $\bz \neq \bz^*$, there must exist two indices $i,j \in [m]$ such that $i \leq a < j$ such that $z^*_i = k-1 > z_i \geq z_j > z^*_j$ due to the constraints on the distributions of zeroes.
This means that we can form another candidate zeroes vector $\bz' = (z'_1, \ldots, z'_m)$ where
\[
z'_s =
\begin{cases}
z_s & \text{if $s \neq \{i,j\}$}\\
z_s + 1 & \text{if $s = i$}\\
z_s - 1 & \text{if $s = j$}
\end{cases}
\]
Now, observe that
\begin{align*}
&\; g(\bz') - g(\bz)\\
= &\; \sum_{s=1}^m \frac{z'_s}{k+1-z'_s} - \sum_{s=1}^m \frac{z_s}{k+1-z_s} \tag{By definition of $g$}\\
= &\; \frac{z'_i}{k+1-z'_i} + \frac{z'_j}{k+1-z'_j} - \frac{z_i}{k+1-z_i} - \frac{z_j}{k+1-z_j} \tag{By definition of $\bz$ and $\bz'$}\\
= &\; \frac{z_i+1}{k-z_i} + \frac{z_j-1}{k+2-z_j} - \frac{z_i}{k+1-z_i} - \frac{z_j}{k+1-z_j} \tag{Since $z'_i = z_i + 1$ and $z'_j = z_j - 1$}\\
= &\; \frac{(z_i+1) (k+1-z_i) - z_i (k-z_i)}{(k-z_i)(k+1-z_i)} + \frac{(z_j-1) (k+1-z_j) - z_j (k+2-z_j)}{(k+2-z_j)(k+1-z_j)} \tag{Grouping the $z_i$ and $z_j$ terms}\\
= &\; \frac{k+1}{(k-z_i)(k+1-z_i)} + \frac{-k-1}{(k+2-z_j)(k+1-z_j)} \tag{Simplifying}
\end{align*}
Now, observe that $k+1 > k > k-1 \geq z_i + 1 > z_i \geq z_j > 0$ implies that ${\color{blue}k+2-z_j \geq k+1-z_i}$ and ${\color{red}k+1-z_j \geq k-z_i}$.
Therefore, $g(\bz') - g(\bz) = (k+1) \cdot \left( \frac{1}{{\color{red}(k-z_i)}{\color{blue}(k+1-z_i)}} - \frac{1}{{\color{blue}(k+2-z_j)}{\color{red}(k+1-z_j)}} \right) \geq 0$
Furthermore, by construction, we see that $\sum_{s=1}^m |z^*_s - z'_s| < \sum_{s=1}^m |z^*_s - z_s|$.
So, by repeating the above argument we can transform any distribution into $\bz^*$ whilst monotonically non-decreasing the function $g$, i.e.\ $\bz^*$ maximizes $g$.
\end{proof}

\prodtosum*
\begin{proof}
We perform induction on $m$.
When $m = 1$, we trivially have $0 \leq a_1 - b_1$.
When $m > 1$, we first observe that
\[
\left( \prod_{i=1}^m a_i \right) - \left( \prod_{i=1}^m b_i \right)
= \left( \left( \prod_{i=1}^{m-1} a_i \right) - \left( \prod_{i=1}^{m-1} b_i \right) \right) \cdot a_m + \left( \prod_{i=1}^{m-1} b_i \right) \cdot \left(a_m - b_m \right)
\]
By induction hypothesis, $\left( \prod_{i=1}^{m-1} a_i \right) - \left( \prod_{i=1}^{m-1} b_i \right) \geq 0$ and so the entire expression is at least $0$ since $a_m \geq b_m \geq 0$.
Meanwhile, $\left( \prod_{i=1}^{m-1} a_i \right) - \left( \prod_{i=1}^{m-1} b_i \right) \leq \sum_{i=1}^{m-1} (a_i - b_i)$ by induction hypothesis and so
\begin{align*}
&\; \left( \prod_{i=1}^m a_i \right) - \left( \prod_{i=1}^m b_i \right)\\
\leq &\; \left( \sum_{i=1}^{m-1} (a_i - b_i) \right) \cdot a_m + \left( \prod_{i=1}^{m-1} b_i \right) \cdot \left(a_m - b_m \right) \tag{By induction hypothesis}\\
\leq &\; \left( \sum_{i=1}^{m-1} (a_i - b_i) \right) + \left(a_m - b_m \right) \tag{Since $a_m \leq 1$, $\prod_{i=1}^{m-1} b_i \leq 1$, and $a_m \geq b_m$}\\
= &\; \sum_{i=1}^m (a_i - b_i)
\qedhere
\end{align*}
\end{proof}

\section{Details of Experiments}
\label{sec:appendix-plots}

\subsection{Generation of Staged Experiment Instances}
\label{sec:appendix-staged-generation}
To generate $\bP$ with $j$ different stage types, we swap the $((s \% j) + 1)$-th action probability with the largest action probability for each stage so that the the optimal action for $s$-th stage always occurs at index $(s \% j) + 1$.
For example, suppose we generated the following SFIPP matrix
\[
\bP =
\begin{bmatrix}
0.124 & 0.357 & \textbf{0.432} & 0.291 & 0.085\\
0.214 & 0.076 & 0.389 & \textbf{0.407} & 0.153\\
0.265 & 0.178 & 0.099 & 0.314 & \textbf{0.348}\\
\textbf{0.428} & 0.067 & 0.209 & 0.134 & 0.275
\end{bmatrix}
\]
where the optimal actions are bolded for each stage.
With $2$ different stage types, we would modify it to
\[
\bP' =
\begin{bmatrix}
{\color{blue}\textbf{0.432}} & 0.357 & {\color{blue}0.124} & 0.291 & 0.085\\
0.214 & {\color{red}\textbf{0.407}} & 0.389 & {\color{red}0.076} & 0.153\\
{\color{blue}\textbf{0.348}}& 0.178 & 0.099 & 0.314 & {\color{blue}0.265} \\
{\color{red}0.067} & {\color{red}\textbf{0.428}} & 0.209 & 0.134 & 0.275
\end{bmatrix}
\]
where the colors indicate the values that are swapped.
After modification, the optimal action for stages $1$ and $3$ will be action $1$ while the optimal action for stages $2$ and $4$ will be action $2$.
In this way, every odd or even stage will be of the same type whilst having differing optimal action indices across types.

\begin{figure}[htb]
    \centering
    \includegraphics[width=0.7\linewidth]{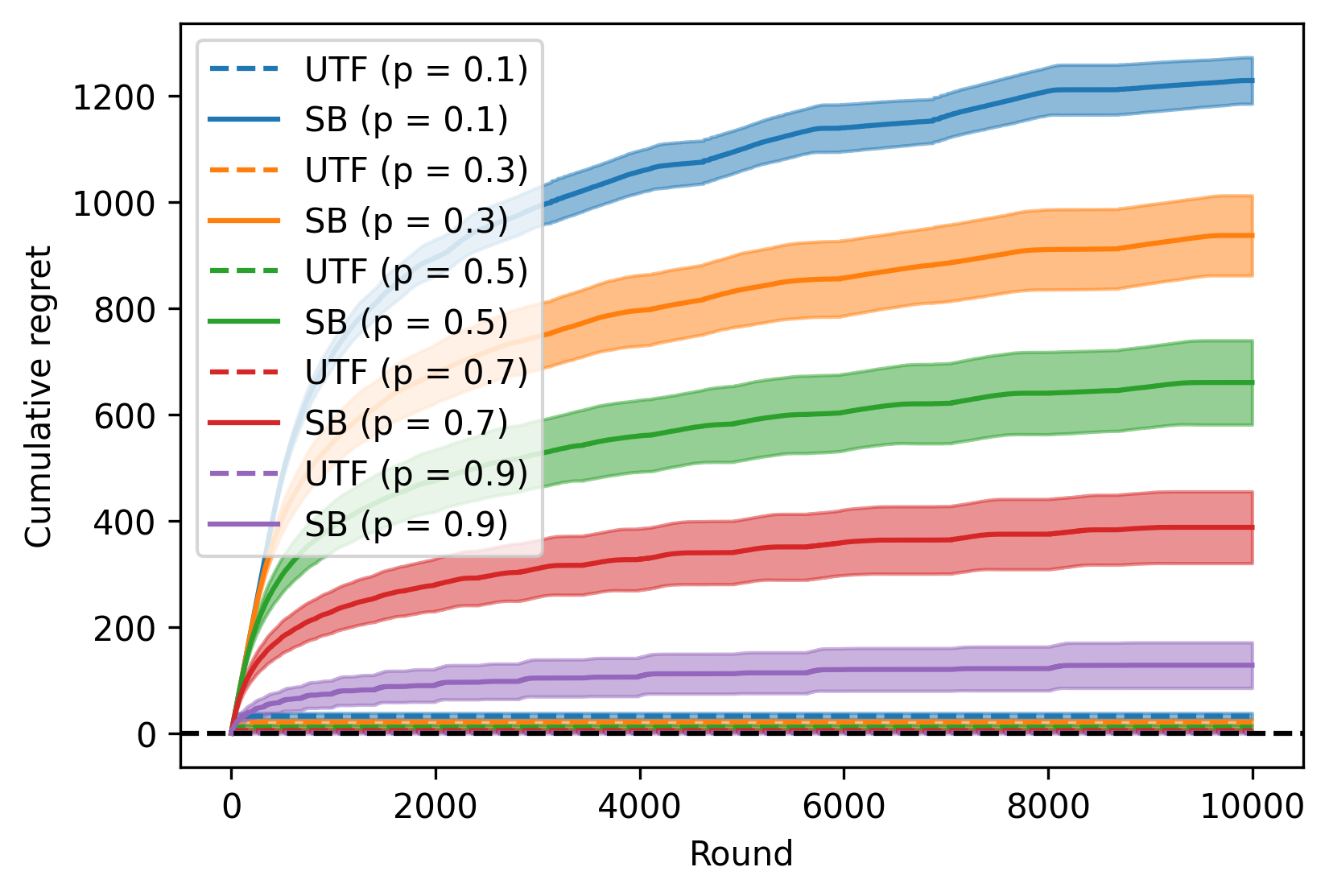}
    \caption{Experimental results on deterministic processes. The bands around the curves represent standard deviations.}
    \label{fig:experiment1-error-bars}
\end{figure}

\begin{figure}[htb]
    \centering
    \includegraphics[width=0.7\linewidth]{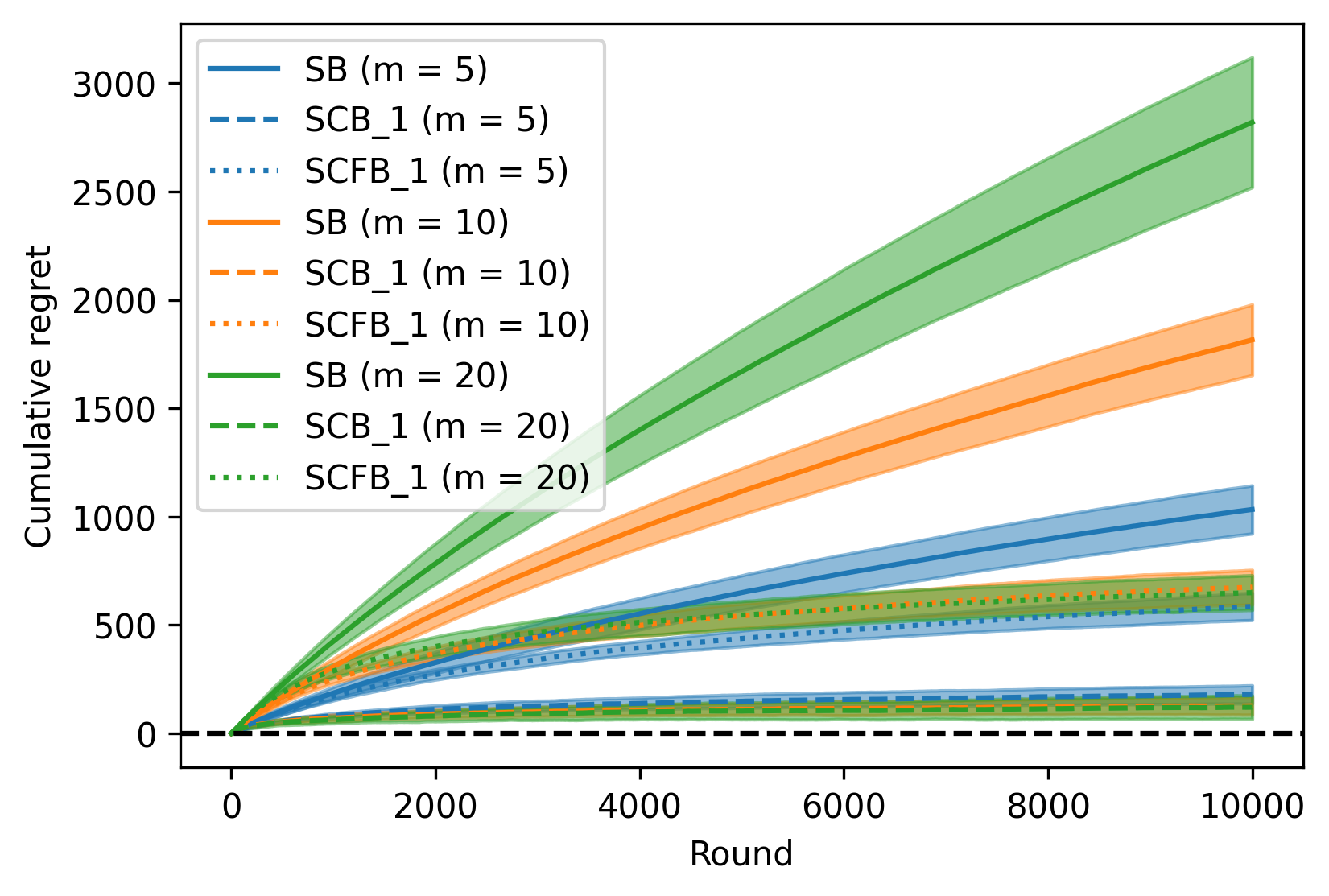}
    \caption{Experimental results on probabilistic processes with one stage type, i.e.\ all stages have the same \emph{unknown} optimal action index, but differing number of stages. Observe that as $m$ grows, the difference in accumulated regret widens between knowing whether to collapse all stages into a single ``meta'' stage or not. The bands around the curves represent standard deviations.}
    \label{fig:experiment2-error-bars}
\end{figure}

\begin{figure}[htb]
    \centering
    \includegraphics[width=0.7\linewidth]{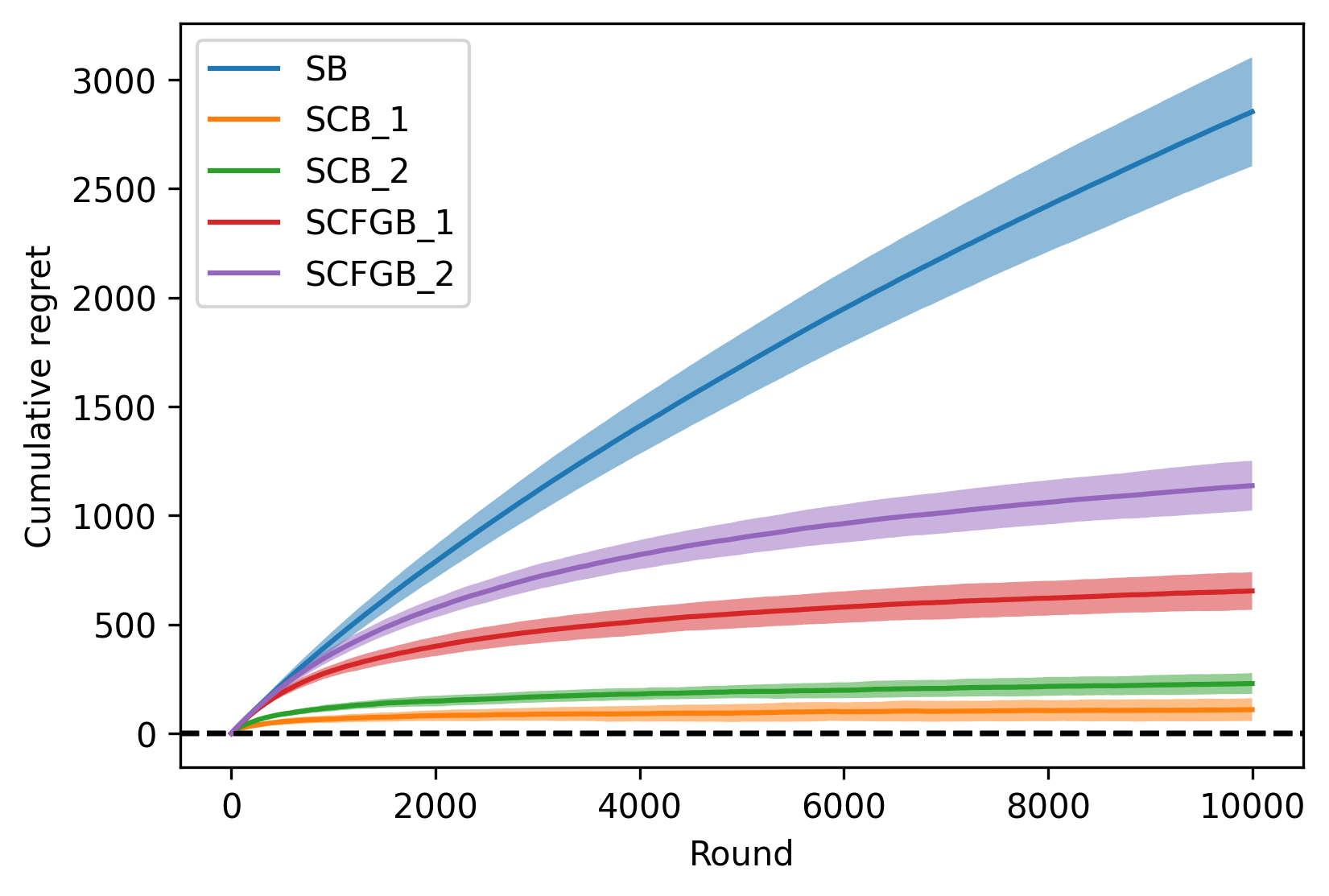}
    \caption{Experimental results on probabilistic processes with one stage type, i.e.\ all stages have the same \emph{unknown} optimal action index. $\bP$ entries are generated from $\mathrm{Beta}(\alpha = 10, \beta = 1)$. Observe that having the information to collapse into one (blue) or two (orange) stages leads to significant improvements in the accumulated regret. The bands around the curves represent standard deviations.}
    \label{fig:experiment3-error-bars}
\end{figure}

\begin{figure}[htb]
    \centering
    \includegraphics[width=0.7\linewidth]{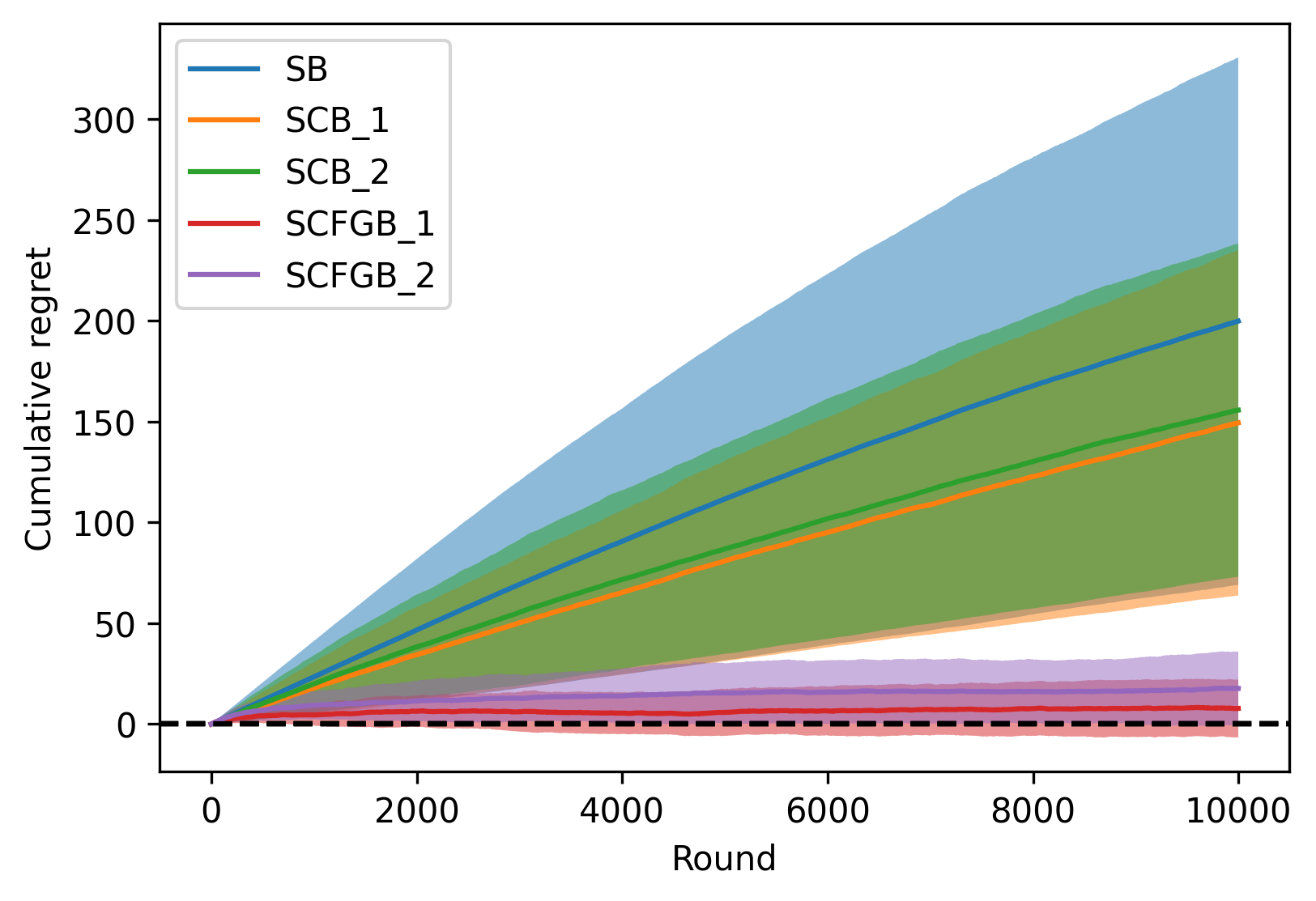}
    \caption{Experimental results on probabilistic processes with one stage type, i.e.\ all stages have the same \emph{unknown} optimal action index. $\bP$ entries are generated from $\mathrm{Beta}(\alpha = 1, \beta = 1)$. Observe that having the information to collapse into one (blue) or two (orange) stages leads to significant improvements in the accumulated regret. The bands around the curves represent standard deviations.}
    \label{fig:experiment3-easy-error-bars}
\end{figure}

\begin{figure}[htb]
    \centering
    \includegraphics[width=0.7\linewidth]{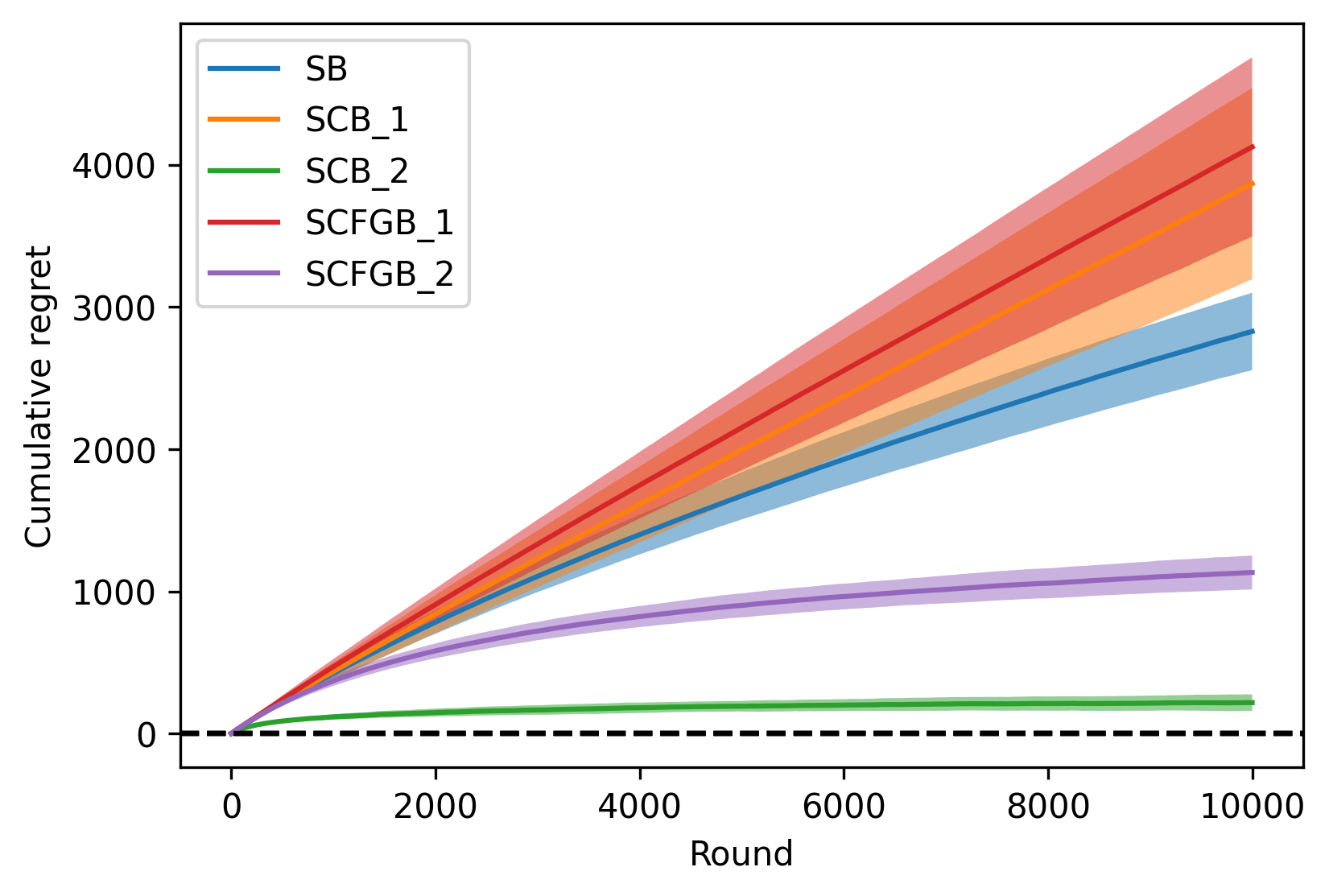}
    \caption{Experimental results on probabilistic processes with two different alternating stage types. $\bP$ entries are generated from $\mathrm{Beta}(\alpha = 10, \beta = 1)$. The bands around the curves represent standard deviations.}
    \label{fig:experiment4-error-bars}
\end{figure}

\begin{figure}[htb]
    \centering
    \includegraphics[width=0.7\linewidth]{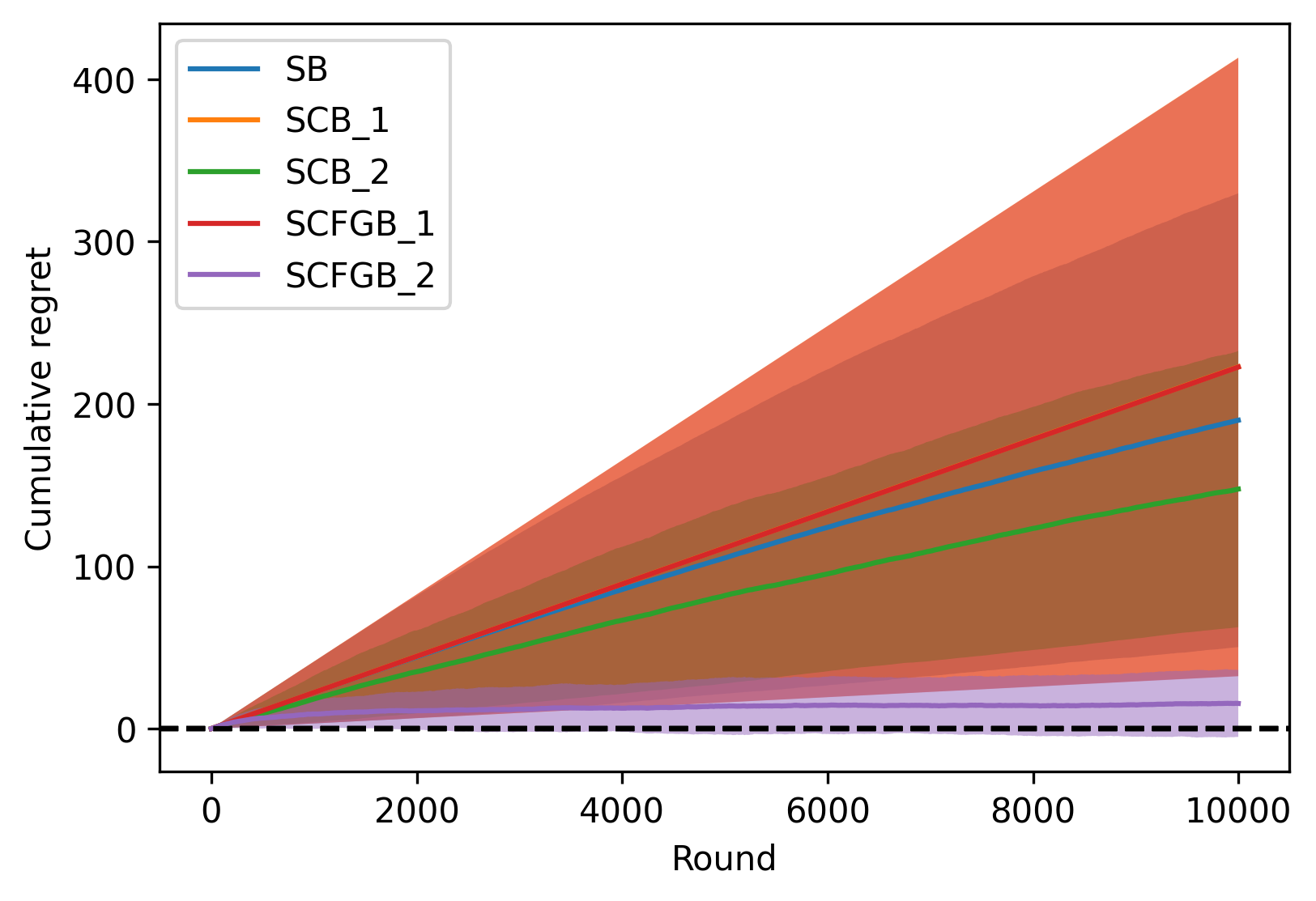}
    \caption{Experimental results on probabilistic processes with two different alternating stage types. $\bP$ entries are generated from $\mathrm{Beta}(\alpha = 1, \beta = 1)$. The bands around the curves represent standard deviations.}
    \label{fig:experiment4-easy-error-bars}
\end{figure}

\end{document}